%% file: main.tex
\documentclass[11pt,letterpaper]{article}
\input{format}
\title{Toward Optimal Switching Regret for Multi-Armed Bandits with Oblivious Adversary}
\author{  Mengxiao Zhang \\
  University of Iowa \\  \texttt{mengxiao-zhang@uiowa.edu} \\}
\date{}
\begin{document}
\begin{titlepage}
\maketitle
\input{sections/00_abstract}
\end{titlepage}
\setcounter{page}{1}
\input{sections/01_introduction}

\input{sections/02_related_work}
\input{sections/03_preliminaries}
\input{sections/04_algorithm}

\input{sections/06_proof}
\input{sections/07_conclusion}
\appendix
\input{sections/a_proofs}
\paragraph{AI methodology statement.}
The author used GPT-6 to generate an initial algorithmic proposal, discuss its analysis, and assist with language editing. The author substantially simplified and revised the proposed algorithm and worked through and refined its analysis. The author determined the paper's organization and extensively rewrote the generated material. The author takes full responsibility for the correctness, originality, and final content of the paper.
\clearpage
\bibliographystyle{alpha}
\bibliography{references}
\end{document}

%% file: format.tex
\input{command}

\renewcommand{\small}{\normalsize}

\SetAlFnt{\normalsize}
\setcitestyle{numbers,square,comma}
\let\citet\citep
\usepackage[normalem]{ulem}

%% file: command.tex
\usepackage[margin=1in]{geometry}
\usepackage[T1]{fontenc}
\usepackage{times}

\let\switchingTextCedilla\c
\usepackage{microtype}
\usepackage{amsthm}

\newtheorem{theorem}{Theorem}[section]

\newtheorem{lemma}[theorem]{Lemma}

\theoremstyle{remark}

\theoremstyle{plain}

\makeatletter
\AtBeginDocument{\patchcmd{\algocf@caption@algo}{\hyper@refstepcounter{algocf}}{}{}{\PackageWarning{switching-template}{Algorithm anchor patch not applied}}}
\makeatother

\AtBeginDocument{\let\switchingMathC\c
	\DeclareRobustCommand{\c}{\ifmmode\expandafter\switchingMathC
		\else\expandafter\switchingTextCedilla\fi}}

\newif\ifspacehack
\usepackage{natbib}
\usepackage{hyperref}
\usepackage{url} 
\usepackage{graphicx}
\usepackage{mathtools}
\usepackage{footnote}
\usepackage{float}
\usepackage{xspace}
\usepackage{multirow}
\usepackage{xcolor}
\usepackage{wrapfig}
\usepackage{framed}
\usepackage{bbm}
\usepackage{footnote}
\usepackage{nicefrac}
\usepackage{makecell}
\usepackage[algo2e,vlined]{algorithm2e} 
\usepackage{algorithm}
\usepackage{amssymb}
\usepackage{bm}
\makesavenoteenv{tabular}
\makesavenoteenv{table}

\newcommand{\calS}{{\mathcal{S}}}

\newcommand{\calD}{{\mathcal{D}}}

\newcommand{\cF}{{\mathcal{F}}}
\newcommand{\cA}{{\mathcal{A}}}
\newcommand{\cX}{{\mathcal{X}}}

\newcommand{\cV}{{\mathcal{V}}}

\newcommand{\Reg}{\mathbf{\mathbf{Reg}}}

\newcommand{\E}{{\mathbb{E}}}

\newcommand{\1}{\mathbb{I}}

\newcommand{\order}{\mathcal{O}}

\newcommand{\Pp}{\mathbb{P}}
\newcommand{\cG}{\mathcal{G}}
\newcommand{\cH}{\mathcal{H}}
\newcommand{\cU}{\mathcal{U}}
\newcommand{\Bern}{\operatorname{Bernoulli}}
\newcommand{\logp}{\log^{+}}
\newcommand{\ip}[2]{\left\langle #1,#2\right\rangle}

\RestyleAlgo{ruled}
\SetAlgoVlined
\DontPrintSemicolon
\SetAlFnt{\small}
\SetAlgoInsideSkip{smallskip}

\newcommand{\Var}{\mathrm{Var}}

\newcommand{\KL}{\text{\rm KL}}

\newcommand{\otil}{\ensuremath{\tilde{\mathcal{O}}}}

\usepackage{lipsum,booktabs}
\usepackage{amsmath,amssymb,amsfonts,bm,enumitem}
\usepackage{rotating}
\usepackage{pdflscape}
\usepackage{hyperref,url}
\hypersetup{
    colorlinks,
    breaklinks,
    linkcolor = blue,
    citecolor = blue,
    urlcolor  = blue,
}
\allowdisplaybreaks
\usepackage{appendix}
\usepackage{multirow,makecell}

\usepackage{algorithmic,algorithm}

\renewcommand{\tilde}{\widetilde}

\def \E {\mathbb{E}}

\def \b {\mathbf{b}}
\def \c {\mathbf{c}}

\usepackage{mathtools}

\usepackage{graphicx,color} 

\definecolor{wine_red}{RGB}{228,48,64}
\definecolor{DSgray}{cmyk}{0,1,0,0}

\usepackage{prettyref}
\newcommand{\pref}[1]{\prettyref{#1}}

\newcommand{\savehyperref}[2]{\texorpdfstring{\hyperref[#1]{#2}}{#2}}
\newrefformat{eq}{\savehyperref{#1}{Eq. \textup{(\ref*{#1})}}}
\newrefformat{eqn}{\savehyperref{#1}{Eq.~(\ref*{#1})}}
\newrefformat{lem}{\savehyperref{#1}{Lemma~\ref*{#1}}}
\newrefformat{def}{\savehyperref{#1}{Definition~\ref*{#1}}}
\newrefformat{line}{\savehyperref{#1}{Line~\ref*{#1}}}
\newrefformat{thm}{\savehyperref{#1}{Theorem~\ref*{#1}}}
\newrefformat{corr}{\savehyperref{#1}{Corollary~\ref*{#1}}}
\newrefformat{cor}{\savehyperref{#1}{Corollary~\ref*{#1}}}
\newrefformat{sec}{\savehyperref{#1}{Section~\ref*{#1}}}
\newrefformat{app}{\savehyperref{#1}{Appendix~\ref*{#1}}}
\newrefformat{assum}{\savehyperref{#1}{Assumption~\ref*{#1}}}
\newrefformat{asm}{\savehyperref{#1}{Assumption~\ref*{#1}}}
\newrefformat{ex}{\savehyperref{#1}{Example~\ref*{#1}}}
\newrefformat{fig}{\savehyperref{#1}{Figure~\ref*{#1}}}
\newrefformat{alg}{\savehyperref{#1}{Algorithm~\ref*{#1}}}
\newrefformat{rem}{\savehyperref{#1}{Remark~\ref*{#1}}}
\newrefformat{conj}{\savehyperref{#1}{Conjecture~\ref*{#1}}}
\newrefformat{prop}{\savehyperref{#1}{Proposition~\ref*{#1}}}
\newrefformat{proto}{\savehyperref{#1}{Protocol~\ref*{#1}}}
\newrefformat{prob}{\savehyperref{#1}{Problem~\ref*{#1}}}
\newrefformat{claim}{\savehyperref{#1}{Claim~\ref*{#1}}}
\newrefformat{que}{\savehyperref{#1}{Question~\ref*{#1}}}
\newrefformat{op}{\savehyperref{#1}{Open Problem~\ref*{#1}}}
\newrefformat{fn}{\savehyperref{#1}{Footnote~\ref*{#1}}}

\def \epsilon {\varepsilon}

%% file: sections/00_abstract.tex
\begin{abstract}
We study switching regret in adversarial multi-armed bandits, where the learner competes with an arm sequence that changes at most $S$ times. When $S$ is known, an optimal expected regret of $\otil(\sqrt{(S+1)KT})$ is obtainable~\citep{auer2002nonstochastic}, where $K$ is the number of arms and $T$ is the horizon length. However, when $S$ is unknown, \citet{marinov2021pareto} show that this guarantee is impossible under an adaptive adversary. In this paper, we show that a single algorithm achieves $\otil(\sqrt{(S+1)KT})$ expected regret for every $S$ against an oblivious adversary, resolving the oblivious case of the open problem posed by \citet{auer2019tracking}. Our algorithm combines a main learner based on EXP3.S~\citep{auer2002nonstochastic}, initialized with a small learning rate, and dyadic-interval subroutines that search for local improvements using randomized learning rates and implicit exploration. Importantly, a non-uniform prior favors following the main learner, keeping the cost of maintaining many subroutines small. When the subroutines accumulate sufficient improvement over the main learner, its learning rate doubles, allowing adaptation to the unknown number of comparator switches $S$.
\end{abstract}

%% file: sections/01_introduction.tex
\section{Introduction}
\label{sec:introduction}

Adversarial multi-armed bandits have been studied for decades as a fundamental model of sequential decision making under partial feedback. Over a horizon of $T$ rounds, a learner repeatedly chooses one of $K$ actions and observes only the loss of the chosen action, while the environment may assign losses to all actions arbitrarily. The classical performance measure is static regret, which compares the learner's cumulative loss with that of the best fixed action in hindsight. The minimax static regret is $\Theta(\sqrt{KT})$~\citep{auer2002nonstochastic,audibert2009minimax}.

Static regret, however, may not be strong enough when the best action changes over time. This motivates the stronger notion of $S$-switching regret, which compares the learner with the best action sequence that switches actions at most $S$ times. In non-stationary stochastic bandits, this benchmark is often motivated by environments whose loss distributions undergo at most $S$ changes. A line of work has developed algorithms that attain near-optimal dynamic regret even without knowing the number of distributional changes $S$ in advance~\citep{auer2019dynamic,auer2019tracking,wei2021nonstationary,suk2022tracking}.

Much less is known when the losses are adversarial. The seminal work of \citet{auer2002nonstochastic} introduced EXP3.S, which achieves $\otil(\sqrt{(S+1)KT})$ switching regret when its parameters are tuned using $S$.\footnote{We use $\otil(\cdot)$ to suppress logarithmic factors.} Thus, the optimal rate is known when the number of switches is given to the learner. For an adaptive adversary, this dependence on prior knowledge cannot in general be removed. In particular, \citet{marinov2021pareto} showed that no single algorithm can achieve the minimax rate simultaneously over all switching classes.

The corresponding problem for an oblivious adversary remained less understood. Recently, \citet{qian2026simultaneous} showed that, for a piecewise-constant deterministic loss sequence with a known number of changes, one algorithm can simultaneously achieve the optimal static and dynamic regret rates. From the perspective of switching regret, their result controls $S=0$ together with one prescribed nonzero value of $S$, equal to the known number of loss changes. It does not provide the optimal guarantee simultaneously for every possible value of $S$. This leaves the following question open:

\begin{center}
\it Can a single algorithm, without knowing $S$, achieve $\otil(\sqrt{(S+1)KT})$ switching regret simultaneously for every $S$ against an oblivious adversary?
\end{center}

\paragraph{Contribution}
In this paper, we answer this question affirmatively by designing a single algorithm that guarantees $\order(\sqrt{(S+1)KT}\log^4 T)$ expected regret simultaneously for all $S\in\{0,\ldots,T-1\}$. This matches the minimax dependence on $S$, $K$, and $T$ up to logarithmic factors, resolving the open problem proposed in~\citet{auer2019tracking}. Our algorithm uses a main learner based on  EXP3.S~\citep{auer2002nonstochastic}, initialized at a small learning rate. A larger learning rate makes this learner more responsive to changing losses but increases its stability cost. Our key idea is to infer when a larger learning rate is needed from local improvements, without estimating the number of switches or locating the switch points. If this main learner falls substantially behind a switching comparator, this disadvantage must appear on dyadic intervals where a fixed arm performs better. We therefore launch auxiliary subroutines, each of which is an instance of online mirror descent (OMD) with entropy regularizer, on dyadic intervals across different time scales. Each auxiliary subroutine can follow the main distribution or test individual arms by shifting probability mass toward them.

Several technical ingredients make this idea work. First, a non-uniform prior places most weight on following the main learner, keeping the aggregate initialization penalty for this option small. Second, randomized learning rates allow the auxiliary subroutines to respond to local improvements whose strength is unknown in advance. Third, because these tests change the sampling distribution, we randomly separate them from updates of the main learner and freeze the main distribution during testing. The auxiliary subroutines use implicit exploration to update from the {} bandit feedback. Finally, from the same feedback, we obtain an observable estimate of their collective improvement over the main learner. When this estimate reaches a prescribed threshold, we restart the main learner at twice its learning rate. Relating these estimates to the actual losses bounds the regret within each epoch and shows that the same improvements that trigger a larger learning rate also control its additional stability cost in expectation.

%% file: sections/02_related_work.tex
\subsection{Related Works}
\label{sec:related-work}

\paragraph{Switching regret in adversarial multi-armed bandits.}
\citet{auer2002nonstochastic} initiated the study of switching regret under bandit feedback through the EXP3.S algorithm. For a comparator with at most $S$ switches, an appropriate choice of its learning rate and mixing parameter yields $\widetilde{\order}(\sqrt{(S+1)KT})$ regret. However, the learning-rate tuning depends on $S$. Their result therefore gives an optimal guarantee for a prescribed number of switches, but it does not provide a single run that is optimally tuned for two numbers of switches requiring different learning-rate scales, let alone for every $S$ simultaneously.

\citet{marinov2021pareto} formulated this simultaneous-adaptation question as a model-selection problem. They characterized the relevant Pareto frontier and showed that an adaptive adversary rules out optimal regret for both a static comparator and richer switching comparators in the same run. However, their lower bound does not cover an oblivious adversary, where the entire loss sequence is fixed before the interaction starts and does not adapt to the learner's actions. In this setting, one generic route is to manage differently tuned copies of EXP3.S using a master construction such as the Bandit-over-Bandit framework of \citet{cheung2021hedging}. However, this only yields $\widetilde{\order}(K^{1/4}T^{3/4}+\sqrt{(S+1)KT})$ switching regret for all $S$ simultaneously~\citep{qian2026simultaneous}. \citet{auer2019tracking} explicitly left optimal adversarial regret without prior knowledge of the number of switches as an open problem.

More recently, \citet{qian2026simultaneous} took a first step toward removing this overhead for oblivious deterministic losses. They obtain optimal static and dynamic regret simultaneously when the loss vector is piecewise constant and the number of its stationary pieces is known. From the perspective of switching regret, their guarantee controls $S=0$ together with one prescribed nonzero value of $S$, determined by the known number of environmental changes. Therefore, it does not give one algorithm that is simultaneously optimal for every number of comparator switches on an arbitrary oblivious loss table.

\paragraph{Dynamic regret beyond adversarial multi-armed bandits.}
Early work on dynamic regret under environmental nonstationarity studied stochastic multi-armed bandits \citep{besbes2014stochastic,auer2019dynamic} and contextual bandits \citep{luo2018efficient,chen2019nonstationary}. This line was extended to linear and generalized linear bandits \citep{cheung2019learning,russac2019weighted,russac2021selfconcordant,faury2021technical}, combinatorial semi-bandits \citep{chen2021combinatorial}, and nonstationary Markov decision processes \citep{cheung2020nonstationarymdp,mao2021nearoptimal}. \citet{wei2021nonstationary} subsequently developed a black-box reduction that recovers the results for multi-armed and contextual bandits and extends optimal dynamic-regret adaptation to linear bandits and reinforcement learning. Further developments in stochastic and contextual bandits include \citet{suk2022tracking,abbasi2023newlook,suk2023tracking}. Related work has also studied bandit convex optimization and nonstationary stochastic optimization with bandit feedback \citep{besbes2015nonstationary,zhao2021bandit,wang2025adaptivity}. These works study environmental nonstationarity or comparator path length, while our result instead adapts to the number of switches of a hindsight comparator on arbitrary oblivious loss tables under finite-armed bandit feedback.

Another recent advance is \citet{rumi2026parameterfree}, which obtains the optimal $\widetilde{\order}(\sqrt{d(S+1)T})$ switching dependence for $d$-dimensional unconstrained adversarial linear bandits without knowing $S$, up to the loss-scale and comparator-norm factors. Their method exploits a scale--direction decomposition specific to the unbounded action space $\mathbb{R}^d$, and they note that extending it to constrained action sets is nontrivial. Our finite-armed bandit setting restricts each played action to one of the $K$ arms and reveals only that arm's loss, so their result does not imply ours.

%% file: sections/03_preliminaries.tex
\section{Preliminaries and Notation}
\label{sec:prelim}

\paragraph{Notation.}
For an integer $n\geq1$, {denote by $[n]\triangleq\{1,\ldots,n\}$ the set of the first $n$ positive integers}. Let \(\Delta_K\) denote the probability simplex on \([K]\), defined by $ \Delta_K \triangleq \{ x\in\mathbb{R}^K: \sum_{i=1}^K x_i=1, x_i\geq 0,~\forall i\in[K] \}. $ Let $e_i$ denote the $i$-th standard basis vector in $\mathbb{R}^K$, and {denote by $\ip{x}{y}$ the standard inner product}. We use bold symbols for sequences{ and denote them by} $\mathbf{a}=\{a_t\}_{t\in[N]}$, where the length $N$ is understood from context. The indicator of an event is denoted by $\1\{\cdot\}$. All logarithms are natural unless a base is displayed, and $\logp z\triangleq\max\{0,\log z\}$ for $z>0$. Let $\mathbf{1}$ and $\mathbf{0}$ denote all-one and all-zero vectors in the appropriate dimension. For probability vectors $x,y$ in a common simplex, let $\KL(x\|y)$ denote the Kullback-Leibler divergence from $x$ to $y$. In our analysis, we also extend the KL divergence to nonnegative vectors whose coordinates sum to at most one. Denote by $\cX\triangleq\operatorname{conv}\{\mathbf0,e_1,\ldots,e_K\}$ the set of such vectors. For $x\in\cX$, denote by $\bar x\triangleq(1-\|x\|_1,x_1,\ldots,x_K)\in\Delta_{K+1}$ its completion, and define $\KL(x\|y)\triangleq\KL(\bar x\|\bar y)$ for $x,y\in\cX$.

\paragraph{Problem setting.}
Assume $T>K\geq2$. The interaction between the learner and the environment lasts for $T$ rounds. Before the learner draws any private randomness, an \emph{oblivious} adversary fixes a deterministic loss table $\ell\in[0,1]^{T\times K}$ where $K$ is the number of actions. {Denote by $\ell_t=(\ell_{t,1},\ldots,\ell_{t,K})$ the loss vector} at round $t$. At each round $t\in[T]$, the learner selects a distribution $p_t\in\Delta_K$ using its past observations and private randomness, draws an arm $i_t\sim p_t$, incurs loss $\ell_{t,i_t}$, and observes only this selected loss. All expectations are over the learner's randomization for the fixed loss table.

\paragraph{Switching regret.}
For $S\in\{0,\ldots,T-1\}$, define the class of comparator sequences with at most $S$ switches by
\begin{align}
\cU_S\triangleq
\left\{
\mathbf{u}=\{u_t\}_{t\in[T]}:
u_t\in\{e_1,\ldots,e_K\},\
\sum_{t=2}^T\1\{u_t\neq u_{t-1}\}\leq S
\right\}.
\label{eq:comparator-class}
\end{align}
For each $S$, fix a best comparator
\begin{align}
\mathbf{u}^{(S)}\triangleq\{u_t^{(S)}\}_{t\in[T]}
\in\arg\min_{\mathbf{u}\in\cU_S}
\sum_{t=1}^T\ip{u_t}{\ell_t},
\label{eq:best-comparator}
\end{align}
breaking ties according to an arbitrary deterministic rule. The comparator $\mathbf{u}^{(S)}$ depends only on the loss table and is therefore fixed before the learner randomizes. For the learner's distribution sequence $\mathbf{p}=\{p_t\}_{t\in[T]}$, define its $S$-switching regret by $ \Reg_T(\mathbf{u}^{(S)}) \triangleq \sum_{t=1}^T\ip{p_t-u_t^{(S)}}{\ell_t}. $ Our goal is to design a single algorithm, without knowledge of $S$, such that $\E\left[\Reg_T(\mathbf{u}^{(S)})\right] \leq \otil\left(\sqrt{(S+1)KT}\right)$ for every $S\in\{0,\ldots,T-1\}$. During the analysis, we also compare other decision sequences with $\mathbf{u}^{(S)}$ over subintervals of the horizon. For an interval $I\subseteq[T]$ and a possibly random sequence $\mathbf{a}=\{a_t\}_{t\in[T]}$ with $a_t\in\Delta_K$, define
\begin{align}
\Reg_I^{\mathbf{a}}(\mathbf{u}^{(S)})
\triangleq
\sum_{t\in I}\ip{a_t-u_t^{(S)}}{\ell_t}.
\label{eq:interval-regret}
\end{align}
In particular, $\Reg_T(\mathbf{u}^{(S)}) =\Reg_{[T]}^{\mathbf{p}}(\mathbf{u}^{(S)})$.

%% file: sections/04_algorithm.tex
\section[Algorithm: EXP3.S with Dyadic Interval Subroutines]{Algorithm: EXP3.S with Dyadic Interval Subroutines}
\label{sec:algorithm}

We now present a single algorithm that, without knowing $S$, achieves $\widetilde{\order}(\sqrt{(S+1)KT})$ expected switching regret simultaneously for all numbers of switches $S$. As mentioned before, our algorithm starts with the main learner based on EXP3.S~\citep{auer2002nonstochastic} at a small learning rate and launches auxiliary subroutines on canonical dyadic intervals to detect local improvements. An observable estimate of their cumulative improvement determines when to restart the main learner at twice its learning rate. In the following, we first explain the algorithm’s intuition, then present its pseudocode and walk through its updates. We conclude with the regret guarantee and a proof overview.

\paragraph{Algorithm intuition.}
To motivate the design of our algorithm, first suppose that $S$ is known. Standard analysis shows that EXP3.S with learning rate $\eta$ admits a regret bound of the form
\begin{align*}
	\order\left(\frac{(S+1)\log T}{\eta}+\eta KT\right),
\end{align*}
see, for example, \citet{auer2002nonstochastic}. The two terms capture a basic tradeoff. A relatively small learning rate limits the stability cost and is appropriate when the comparator rarely switches, but it may respond too slowly when the comparator changes frequently. A large learning rate is more responsive, but incurs a greater stability cost. If $S$ were known, these two effects could be balanced in advance. Because $S$ describes a hindsight comparator, however, the learner cannot determine the appropriate scale during play.

Our key idea is to infer the need for greater responsiveness without estimating $S$ or locating the switch points. Instead, the algorithm looks for a locally estimable signal of an overly conservative learning rate: over some period, a fixed arm performs substantially better than the main learner. This leads to the central design principle of our approach: \emph{search for local regret witnesses}.

Indeed, a comparator with at most $S$ switches plays a fixed arm on at most $S+1$ contiguous segments. If the main learner suffers large switching regret, this regret must therefore be reflected in local periods over which some fixed arm has a substantial advantage. However, as the locations and lengths of these periods are unknown, the algorithm then must search over many time scales simultaneously. Canonical dyadic intervals provide an efficient multiscale structure: every comparator segment can be decomposed into only a logarithmic number of such intervals. We therefore launch an auxiliary OMD subroutine on each dyadic interval. Each auxiliary subroutine tests whether shifting probability mass from the main distribution toward a particular arm would improve performance. Without reconstructing the comparator, the subroutines collectively search for \emph{local regret witnesses}. The accumulated evidence from these witnesses provides a criterion for doubling the main learner's learning rate.

Several issues of this approach are as follows. First, we may launch too many subroutines such that their initialization cost already dominates the overall regret. To handle this, we initialize each subroutine with \emph{a non-uniform prior}: it places mass $T/(T+1)$ on following the main distribution and spreads the remaining mass uniformly over the $K$ arms. Each auxiliary subroutine therefore starts with a preference for the main distribution, while preserving the ability to shift toward a better arm. Technically, this asymmetry keeps the initialization cost of the many subroutines that do not yield useful regret witnesses small, while allowing strong witnesses to exceed the cost of comparing with an alternative arm. The second issue is how to choose a subroutine's learning rate before the strength of its local improvement is known. Large rates improve responsiveness but increase the cost of learning from bandit feedback. Our solution is to randomize the learning rates so that stronger improvements are more likely to receive a sufficiently responsive subroutine, while controlling the aggregate estimation and update costs in expectation.

To decide when to increase the main learning rate, we estimate the cumulative improvement produced by these local tests. To avoid interaction between the auxiliary tests and updates of the main distribution, we also randomly separate the following two roles. On \emph{main rounds}, only the main learner, which is based on EXP3.S, updates. On \emph{challenge rounds}, which are reserved for testing, the main learner's distribution is frozen while the active interval subroutines perturb the sampling distribution. The resulting bandit feedback allows us to estimate the loss improvement of the perturbed sampling distribution over the main distribution. We accumulate these estimates since the last restart and call their sum \emph{credit}.

Starting from a conservative rate, the algorithm then restarts the main learner at twice its learning rate when the above mentioned credit reaches a prescribed threshold. Intuitively, this credit serves as a signal that the current rate is too small and also an estimate of the regret reduction that controls the additional cost of using a larger rate. The algorithm therefore uses these local tests to adapt its learning rate without knowing the comparator's number of switches or locating the switch points.

\begin{algorithm2e}[t]
\caption[EXP3.S with dyadic interval subroutines]{EXP3.S  with dyadic interval subroutines}
\label{alg:main}
\DontPrintSemicolon
\LinesNumbered
\KwIn{Number of arms $K$ and horizon $T$.}
Set $L\triangleq\left\lceil{8\log T}\right\rceil$,  $\eta_1=\frac{{L}}{\sqrt{KT}}$, {and} $\alpha=\frac{1}{{L^2}}${}.\;

Set $j=1$, $\eta_j=\eta_1$, $C_j=0$, $\cA=\emptyset$, and $q_{1,i}=1/K$ for all $i\in[K]$.\;
\For{$t=1,\ldots,T$}{
  \ForEach{canonical dyadic interval $J$ (defined in \pref{eq:canonical-dyadic}) starting at $t$}{
    Draw $U_J\sim\operatorname{Unif}(0,1)$ and set $\eta_J=\alpha\eta_j/U_J$. \nllabel{line:interval-rate}\;
    Set $x_{t,J,i}=1/(K(T+1))$ for all $i\in[K]$ and add $J$ to $\cA$. \nllabel{line:interval-init}\;
  }
  Set $\cA_t=\cA$ and draw $b_t\sim\Bern(1/2)$. \nllabel{line:routing}\;
  \lIf{$b_t=1$}{
    Set $p_t=q_t$.
  }
  \lElse{
    Set $p_t$ according to \pref{eq:challenge-distribution}. \nllabel{line:challenge-mixture}
  }
  Draw $i_t\sim p_t$ and observe $\ell_{t,i_t}$.\;
  \eIf{$b_t=1$}{
    Set $\widehat\ell^q_{t,i}=2\1\{i=i_t\}\ell_{t,i_t}/q_{t,i_t}$ for all $i$.\;
    Compute $\widetilde q_{t+1}=\arg\min_{q\in\Delta_K}\{\eta_j\ip{q}{\widehat\ell_t^q}+\KL(q\|q_t)\}$.\;
    Set $q_{t+1,i}=(1-1/T)\widetilde q_{t+1,i}+1/(KT)$ for all $i$. \nllabel{line:fixed-share-update}\;
    Leave all interval states unchanged and set $Z_t=0$.\;
  }{
    Set $q_{t+1}=q_t$. \nllabel{line:freeze-main}\;
    \ForEach{$J\in\cA_t$}{
      Set $\widehat\ell_{t,J,i}=\1\{i=i_t\}\ell_{t,i_t}/(p_{t,i_t}+\eta_J)$ for all $i$.\;
      Set $z_{t,J,i}=\widehat\ell_{t,J,i}-\ip{q_t}{\widehat\ell_{t,J}}$ for all $i$.\;
      Set $x_{t+1,J}=\arg\min_{x\in\cX}\{\eta_J\ip{x}{z_{t,J}}+\KL(x\|x_{t,J})\}$. \nllabel{line:interval-update}\;
    }
    Set $Z_t=\ell_{t,i_t}(q_{t,i_t}/p_{t,i_t}-1)$ \nllabel{line:credit-sample}.\;
  }
  Set $C_j\gets C_j+Z_t$. \nllabel{line:credit-account}\;
  \lIf{$C_j\geq {4}KT\eta_j$ and $t<T$}{
   $j\gets j+1$, $\eta_j\gets2\eta_{j-1}$, $C_j\gets0$, $\cA\gets\emptyset$, and $q_{t+1,i}\gets1/K$ for all $i$. \nllabel{line:restart}
}
  \lElse{
    Remove from $\cA$ every $J\in\cA_t$ ending at round $t$.
  }
}
\end{algorithm2e}

\paragraph{Algorithm walkthrough.}
We now follow \pref{alg:main} in its execution order while explaining how each step implements the preceding intuition. Throughout the algorithm, $q_t\in\Delta_K$ denotes the distribution maintained by the main learner, whereas $p_t\in\Delta_K$ denotes the actual distribution from which the learner draws its arm. The main learner provides the reference distribution, the interval subroutines test for local improvements over this reference, and we use the detected improvement to determine when the main learning rate should be increased.

The algorithm proceeds in consecutive random epochs. During epoch $I_j$, the main learner uses the fixed learning rate $\eta_j$. We set {$L\triangleq\left\lceil{8\log T}\right\rceil$, $\alpha\triangleq\frac{1}{{L^2}}$, and $\eta_1\triangleq\frac{{L}}{\sqrt{KT}}$}, as defined in \pref{alg:main}. The algorithm starts with the conservative rate $\eta_1$ and the uniform main distribution. The rate remains fixed within each epoch.

At each round $t$, the algorithm initializes an auxiliary subroutine on each canonical dyadic interval beginning at $t$. Concretely, for integers $h,k\geq0$ satisfying $(k+1)2^h\leq T$, we call
\begin{align}\label{eq:canonical-dyadic}
    J_{k,h}\triangleq\{k2^h+1,\ldots,(k+1)2^h\}
\end{align}
a \emph{canonical dyadic interval} of length $2^h$. In the following, we omit the subscript $k$ and $h$, and use $J$ to denote a canonical dyadic interval in general. These intervals are defined using the original global time index and are not shifted when a new epoch begins. In \pref{line:interval-rate} and \pref{line:interval-init}, the algorithm launches a subroutine for every canonical dyadic interval beginning at the current round. Since every constant segment of a switching comparator can be decomposed into only $\order(\log T)$ such intervals, a large switching regret must be witnessed by a controlled number of subroutines. This is how the algorithm implements the principle of searching for local regret witnesses.

More concretely, each subroutine is an OMD algorithm with entropy regularizer that maintains $x_{t,J}\in\cX$, where $\cX$ is defined in \pref{sec:prelim}. Specifically, the coordinate $x_{t,J,i}$ is the mass assigned to challenging $q_t$ with arm $i$, whereas the missing mass $1-\|x_{t,J}\|_1$ represents the probability of following the main distribution. This additional option allows a subroutine to remain close to $q_t$ when it has not found convincing evidence for any arm. Crucially, \pref{line:interval-init} uses a strongly non-uniform initialization: $x_{t,J,i}=1/(K(T+1))$, leaving mass $T/(T+1)$ on following the main distribution. This addresses the first issue in the intuition by keeping the aggregate initialization cost small. Intuitively, this prior makes following the main learner the natural starting point for each auxiliary subroutine. Many intervals may offer no useful improvement, so making it inexpensive to stay with the main distribution helps control the cost of searching across many time scales. At the same time, every arm still has an $\Omega(1/(KT))$ initial weight, allowing an auxiliary subroutine to shift toward a better arm when the feedback supports doing so. At launch, \pref{line:interval-rate} draws $U_J\sim\operatorname{Unif}(0,1)$ and assigns the rate $\eta_J=\alpha\eta_j/U_J$. We first describe how the subroutine is used and return to the reason for this randomized choice at the end of the walkthrough.

After launching the new subroutines, \pref{line:routing} draws a fresh coin $b_t\sim\Bern(1/2)$ that determines how the observed loss is used. On a main round, where $b_t=1$, the learner sets $p_t=q_t$. The observed loss is converted into an importance-weighted loss vector and passed only to the main learner, while the interval states remain unchanged. The factor $2$ in the estimator compensates for entering the main branch with probability $1/2$, making the estimator unbiased. The main learner's EXP3.S update in \pref{line:fixed-share-update} mixes the entropy-OMD iterate with a small amount of the uniform distribution. The mixed distribution is used as the reference point in the next OMD update, and every arm has probability at least $1/(KT)$.

On a challenge round, where $b_t=0$, the main distribution is instead frozen in \pref{line:freeze-main}. The active subroutines perturb the actual distribution according to
\begin{align}
p_t
=
\left(
1-\alpha\sum_{J\in\cA_t}\|x_{t,J}\|_1
\right)q_t
+
\alpha\sum_{J\in\cA_t}x_{t,J},
\label{eq:challenge-distribution}
\end{align}
where $\cA_t$ is the collection of active subroutines. Thus each subroutine transfers a small amount of probability mass from $q_t$ toward the arms it currently considers promising, while $\alpha$ limits the aggregate influence of all auxiliary subroutines. Using only main-round observations in the main learner's updates keeps these updates separate from the auxiliary tests, as described in the intuition. On these challenge rounds, after drawing $i_t\sim p_t$ and observing $\ell_{t,i_t}$, the algorithm sends this single bandit observation to every active subroutine through a subroutine-specific implicit-exploration estimate. Because the arm is sampled from the {} actual distribution $p_t$, which can differ from the distribution proposed by an individual subroutine, it is important to have the implicit-exploration offset $\eta_J$ to prevent the estimate from becoming excessively large when this probability is small. The estimated loss is then centered by the estimated loss of $q_t$ before the OMD update in \pref{line:interval-update}. Consequently, the subroutine learns whether its proposed perturbation improves upon the main distribution.

The improvement estimate described in the intuition is accumulated in the credit variable $C_j$, initialized to zero at the start of epoch $I_j$. In \pref{line:credit-sample}, the algorithm records $Z_t=\ell_{t,i_t}(q_{t,i_t}/p_{t,i_t}-1)$ on a challenge round and sets $Z_t=0$ on a main round. Conditional on the distributions chosen before the arm is drawn, the mean of $Z_t$ is $\ip{q_t-p_t}{\ell_t}$. A positive $\ip{q_t-p_t}{\ell_t}$ therefore indicates that $p_t$ has smaller expected loss than $q_t$. Moreover, this credit enters with a negative sign when the actual regret is decomposed into the regret of $q_t$ and the difference between $p_t$ and $q_t$. It is therefore both evidence that the current main learner is insufficiently responsive and a way to control the cost of increasing its learning rate.

The update in \pref{line:credit-account} then adds $Z_t$ to the current epoch credit $C_j$. Once $C_j\geq {4}KT\eta_j$, this indicates that the actual decision distribution has accumulated enough improvement over the main distribution to control the additional stability cost associated with doubling the learning rate. Therefore, \pref{line:restart} terminates epoch $I_j$, discards its active subroutines, resets the main distribution, and starts epoch $I_{j+1}$ with zero credit and learning rate $\eta_{j+1}=2\eta_j$. Thus the algorithm becomes more responsive only after the auxiliary subroutines have accumulated sufficient improvement to support this change. \par
It remains to explain the randomized learning rate in \pref{line:interval-rate}, which addresses the second issue in the intuition: the strength of a local improvement is unknown when its subroutine starts. For an arm $a$, define its cumulative advantage over the main distribution on $J$ by $A_J(a)\triangleq\sum_{t\in J:b_t=0} \ip{q_t-e_a}{\ell_t}$. A positive value means that arm $a$ incurs less loss than $q_t$ on the challenge rounds of $J$.  Recall that the purpose of the subroutine on $J$ is to test whether shifting probability mass from $q_t$ toward a fixed arm can improve the learner's performance on $J$, without aiming to solve a separate bandit problem on its own. At a high level, $\eta_J$ determines how aggressively subroutine $J$ tests its alternatives. A larger rate allows it to move probability mass away from $q_t$ more quickly and hence react to a weaker or shorter-lived advantage. Such responsiveness is also more costly, because the subroutine becomes more sensitive to noisy bandit estimates. A smaller rate keeps the subroutine closer to $q_t$ and is less costly, but requires a stronger or more persistent advantage before it reacts substantially. Since $A_J(a)$ is unknown when $J$ begins, the algorithm cannot select the appropriate sensitivity directly.

Consequently, the draw $U_J\sim\operatorname{Unif}(0,1)$ assigns each interval a random sensitivity level. If an interval contains a strong regret witness, then a broad range of rate realizations can respond to it. A weaker witness requires a more aggressive realization, but it also contributes less to the total regret. The analysis therefore does not require every subroutine to receive the appropriate rate. Instead, it aggregates the credit generated by the subroutines that happen to be sufficiently responsive and uses their contributions to control the total advantage over the dyadic cover. At the same time, using the same $\eta_J$ as the implicit-exploration offset limits the estimation cost of aggressive subroutines.

We can now state the switching-regret guarantee of \pref{alg:main}.

\begin{theorem}
\label{thm:main}
For every oblivious loss table $\ell\in[0,1]^{T\times K}$, and every $S\in\{0,\ldots,T-1\}$, \pref{alg:main} guarantees
\begin{align}
\E\left[\Reg_T(\mathbf{u}^{(S)})\right]
\leq
\order\left(\sqrt{(S+1)KT}\log^4 T\right).
\label{eq:main-theorem}
\end{align}
\end{theorem}

To our knowledge, \pref{alg:main} is the first algorithm for finite-armed adversarial bandits that, against every oblivious loss table, achieves the minimax-optimal dependence $\widetilde{\order}(\sqrt{(S+1)KT})$ simultaneously for every $S\in\{0,\ldots,T-1\}$ without prior knowledge of $S$. This resolves the adversarial switching-regret problem posed by \citet{auer2019tracking} and subsequently highlighted by \citet{marinov2021pareto,luo2022corralling,qian2026simultaneous}.

\paragraph{Proof overview.}
Let $I_1,\ldots,I_N$ denote the epochs reached by the algorithm, let $H_j\triangleq |I_j|$, and let $C_j\triangleq\sum_{t\in I_j}Z_t$ be the credit accumulated within epoch $I_j$. {We denote by $\E_j[\cdot]$ the expectation conditional on the} history before $I_j$ begins. Since $Z_t$ has conditional mean $\ip{q_t-p_t}{\ell_t}$ given the information available before sampling $i_t$, we have
\begin{align}
\E\left[\Reg_T(\mathbf{u}^{(S)})\right]
=
\E\left[
\sum_{j=1}^N
\Reg_{I_j}^{\mathbf{q}}(\mathbf{u}^{(S)})
-
\sum_{j=1}^N C_j
\right].
\label{eq:proof-overview-decomposition}
\end{align}
We therefore bound the regret of the main sequence $\{q_t\}_{t\in[T]}$ while keeping the negative credit. We use two complementary bounds within each epoch. Via a standard EXP3.S analysis, we know that
\begin{align}\label{eq:large-learning-rate}
\E_j[\Reg_{I_j}^{\mathbf{q}}(\mathbf{u}^{(S)})]
\leq
\widetilde{\order}((S+1)/\eta_j)
+\eta_jK\E_j[H_j].    
\end{align}
The inverse-rate term decreases with $\eta_j$, whereas the stability term will be controlled by credit accumulated in completed epochs. This bound is therefore useful when $\eta_j$ is large. When $\eta_j$ is small, however, its inverse-rate term can be much larger than the desired regret.

For the small-rate epochs, we instead use the interval subroutines to control the regret of the main sequence directly, which is the key technical ingredient enabling adaptation to an unknown number of switches. Specifically, we show that
\begin{align}
\E_j\left[
\Reg_{I_j}^{\mathbf{q}}(\mathbf{u}^{(S)})
\right]
\leq
\order\left(
\eta_jKTL^2+\sqrt{(S+1)KT}\,L^3
\right),
\label{eq:proof-overview-small-rate}
\end{align}
which has no terms inversely related to $\eta_j$. To prove \pref{eq:proof-overview-small-rate}, we first describe the comparison at a deterministic endpoint $n$, with the main learning rate fixed at $\eta_j$. Denote by $\calD_j(n)$ the corresponding comparator-dependent dyadic cover, which has at most $(S+1)L$ intervals. For each dyadic interval $J\in\calD_j(n)$, denote by $a_J$ the fixed arm played by the comparator throughout $J$. Denote its challenge-round advantage through $n$ by $A_J(n)\triangleq\sum_{\substack{t\in J,\ t\leq n\\b_t=0}}\ip{q_t-e_{a_J}}{\ell_t}$.  In the analysis, denote the total scaled advantage on intervals with sufficiently large learning rates by $W_j(n)\triangleq\alpha\sum_{J\in\calD_j(n)}A_J(n)\1\{\eta_JA_J(n)\geq4L\}$.

We next relate $W_j(n)$ to the total advantage over the cover. The rate needed to meet the condition defining $W_j(n)$ depends on $A_J(n)$, which is unknown when $J$ begins. For advantages fixed independently of the rate draws, $\eta_J=\alpha\eta_j/U_J$ guarantees that intervals with $A_J(n)\geq4L/(\alpha\eta_j)$ always meet the threshold, so their total advantage is at most $W_j(n)/\alpha$. Smaller positive advantages meet the threshold with probability proportional to $A_J(n)$, giving quadratic expected contributions to $W_j(n)$. Using these two cases, we show that, with high probability,
\begin{align}
\sum_{J\in\calD_j(n)}A_J(n)
\leq
\order\left(
\frac{W_j(n)+L\sqrt{\frac{(S+1)W_j(n)}{\eta_j}}+\frac{\sqrt{S+1}L^2}{\eta_j}}{\alpha}
\right).
\label{eq:overview-advantage-recovery}
\end{align}

To further bound $W_j(n)$ in \pref{eq:overview-advantage-recovery}, we decompose it into the auxiliary subroutines' collective improvement and their total weighted regret. Credit estimates the improvement term, while OMD bounds the subroutines' regret against the comparator arms and the main learner. For comparison with $a_J$, the initialization penalty and the comparator arm's estimation error under implicit exploration together are at most $L/\eta_J$ with high probability. The condition $\eta_JA_J(n)\geq4L$ therefore bounds their total weighted contribution by a constant fraction of $W_j(n)$. For the other subroutines, we compare with the main learner. This is why the non-uniform prior places most weight on following it, keeping the total initialization penalty for these comparisons small. Implicit exploration controls the errors in estimating the main learner's losses, the correction relating the estimates to credit, and the OMD update terms. Averaging these bounds over the randomized rates gives a total expected contribution of $\order(KT\eta_j)$. We can therefore bound $W_j(n)$ using credit at the epoch endpoint, where the restart threshold bounds credit.

One technical issue remains: the epoch endpoint $\rho_j$ is determined by the credit process and therefore depends on the randomized subroutines. We handle this by analyzing a fixed-rate continuation without further restarts. The separation of main updates and auxiliary tests then gives the independence used above. Denote by $C_{j,n}^{\mathrm{cont}}$ the credit through $n$ in this continuation. We establish both comparisons simultaneously over deterministic prefixes before evaluating them at $\rho_j$. Taking expectations and using the cost bound above gives
\begin{align}
\E_j[C_{j,\rho_j}^{\mathrm{cont}}]
\geq\frac34\E_j[W_j(\rho_j)]-\order(KT\eta_j).
\label{eq:overview-selected-credit}
\end{align}
See \pref{lem:selected-credit-recovery} for details. The restart rule further gives $C_{j,\rho_j}^{\mathrm{cont}}=C_j\leq {4}KT\eta_j+1$. Combining this upper bound with \pref{eq:overview-selected-credit} yields $\E_j[W_j(\rho_j)]=\order(KT\eta_j)$. Taking expectations in \pref{eq:overview-advantage-recovery} therefore gives
\begin{align*}
\E_j\left[\sum_{J\in\calD_j(\rho_j)}A_J(\rho_j)\right]
&\leq\order\left(\frac{KT\eta_j+L\sqrt{(S+1)KT}+\frac{\sqrt{S+1}L^2}{\eta_j}}{\alpha}\right)\\
&\leq\order\left(\eta_jKTL^2+\sqrt{(S+1)KT}\,L^3\right),
\end{align*}
where the second inequality uses $\alpha=1/({L^2})$ and $\eta_j\geq{L}/\sqrt{KT}$. The fair routing of $b_t$ makes the expected main-sequence regret twice this expected challenge-round advantage, proving \pref{eq:proof-overview-small-rate}.

Finally, apply \pref{eq:proof-overview-small-rate} when $\eta_j\leq\eta_1\sqrt{S+1}$ and \pref{eq:large-learning-rate} otherwise. Completed epochs satisfy $C_j\geq {4}KT\eta_j$. However, the final epoch can have negative credit, which increases regret in \pref{eq:proof-overview-decomposition}. Fortunately, the uniform prefix bound underlying \pref{eq:overview-selected-credit}, together with $W_j(n)\geq0$, bounds the expected maximum amount by which credit falls below zero within each epoch by $\order(KT\eta_j)$. With the rates doubling, credit accumulated in completed epochs therefore offsets both the $\eta_jKH_j$ terms and this possible negative credit in expectation, up to a cost at the initial rate. Specifically, we have
\begin{align*}
\E\left[\sum_{j=1}^N\eta_jKH_j-\sum_{j=1}^NC_j\right]
\leq{2}KT\eta_1.
\end{align*}
Summing \pref{eq:proof-overview-small-rate} over the small-rate epochs adds at most a factor $\order(\log T)$. For the large-rate epochs, the inverse-rate terms in \pref{eq:large-learning-rate} halve at each restart, so their sum is at most twice the first term. Substituting these sums and the credit bound into \pref{eq:proof-overview-decomposition}, with $\eta_1={L}/\sqrt{KT}$, gives $\order(\sqrt{(S+1)KT}\log^4 T)$ expected regret, proving \pref{thm:main}.

%% file: sections/06_proof.tex
\section{Proof of \pref{thm:main}}
\label{sec:proof}

In this section, we prove \pref{thm:main} following the proof overview. Fix $S\in\{0,\ldots,T-1\}$ and the deterministic comparator $\mathbf{u}^{(S)}$ determined by the oblivious loss table. {For every reached epoch, denote its interval and length by} $I_j=\{\tau_j,\ldots,\rho_j\}$ and $H_j=\rho_j-\tau_j+1$, where $\tau_j$ and $\rho_j$ are its first and last rounds, respectively. Let $\cF_j$ denote the history before round $\tau_j$, {and denote conditional expectation by} $\E_j[\cdot]\triangleq\E[\cdot\mid\cF_j]$. Let $N$ be the number of reached epochs. All sums over epochs below include only reached epochs.

We use two within-round histories. Let $\cH_t$ contain all randomness revealed before drawing $b_t$, including the rates of newly launched intervals, and let $\cG_t=\sigma(\cH_t,b_t)$. Thus $q_t$ is $\cH_t$-measurable, while $p_t$ is $\cG_t$-measurable and $i_t\sim p_t$ conditional on $\cG_t$. For the interval updates on $\cX$, we use the completed relative entropy defined in \pref{sec:prelim}.

At most one canonical interval of each length is active at a given round, so $|\cA_t|\leq L$. Since $\alpha L\leq1/2$, the mixture in \pref{eq:challenge-distribution} gives $p_t\geq q_t/2$ coordinatewise on challenge rounds. Consequently, $|Z_t|\leq1$. We will use this bound to control the overshoot at an epoch endpoint.

\paragraph{Step 1: decompose the regret across epochs.}
For every reached epoch, define $C_j\triangleq\sum_{t\in I_j}Z_t$ to be the cumulative credit at the end of epoch $j$. Conditional on $\cG_t$, the distributions $p_t$ and $q_t$ are fixed and $i_t\sim p_t$. Hence, on a challenge round, we have
\begin{align*}
\E[Z_t\mid\cG_t]
&=
\sum_{i=1}^Kp_{t,i}\ell_{t,i}
\left(\frac{q_{t,i}}{p_{t,i}}-1\right)
=
\ip{q_t-p_t}{\ell_t}.
\end{align*}
On a main round, both sides are zero because $p_t=q_t$ and $Z_t=0$. Summing this identity over the horizon and using the partition into epochs gives
\begin{align}
\E\left[\Reg_T(\mathbf{u}^{(S)})\right]
=
\E\left[
\sum_{j=1}^N
\Reg_{I_j}^{\mathbf q}(\mathbf{u}^{(S)})
-
\sum_{j=1}^NC_j
\right].
\label{eq:main-regret-decomposition}
\end{align}
In the following, we bound the main-sequence regret in Step~2 and use credit in Step~3 to offset the terms proportional to $\eta_j$.

\paragraph{Step 2: derive two guarantees for the main sequence $\{q_t\}_{t\in[T]}$.}
The first guarantee follows from the standard EXP3.S analysis for the main learner's update, together with the loss estimates that reflect the random routing.

\begin{lemma}
\label{lem:fixed-share}
For every reached epoch $I_j$, \pref{alg:main} guarantees that
\begin{align}
\E_j\left[
\Reg_{I_j}^{\mathbf q}(\mathbf{u}^{(S)})
\right]
\leq
\frac{3(S+1)\log T}{\eta_j}
+
\eta_jK\E_j[H_j].
\label{eq:fixed-share-epoch}
\end{align}
\end{lemma}

The proof is deferred to \pref{app:fixed-share}. However, the first term in \pref{eq:fixed-share-epoch} is too large when $\eta_j$ is small, which ruins the guarantee in early epochs. We therefore utilize the guarantee obtained via interval subroutines when $\eta_j$ is small. Specifically, we will bound the total dyadic advantage using the randomized rates, then relate this advantage to credit.

Before stating the small-rate bound, we define the continuation credit and the dyadic intervals used in the analysis. Fix a reached epoch $I_j$. As mentioned before, since the endpoint of epoch $I_j$ is random, we need to define an auxiliary process, called \emph{fixed-rate continuation}, starting from the algorithm's state at the beginning of round $\tau_j$. This auxiliary process keeps the main learning rate equal to $\eta_j$ and suppresses all subsequent restarts until the global time $T$. It uses the same random variables as the actual algorithm through round $\rho_j$ and fresh continuation randomness afterward. Since the adversary is oblivious, both processes run against the same fixed loss table, including after $\rho_j$. Let $Z_t^{(j,\mathrm{cont})}$ be the credit sample generated by this auxiliary process at round $t$, and, for every $n\in\{\tau_j,\ldots,T\}$, define $C_{j,n}^{\mathrm{cont}}\triangleq \sum_{t=\tau_j}^n Z_t^{(j,\mathrm{cont})}$. By construction, the continuation and the actual algorithm coincide on $I_j$, so $C_{j,\rho_j}^{\mathrm{cont}} =\sum_{t=\tau_j}^{\rho_j}Z_t=C_j$. Let $\mathcal J$ denote the fixed family of canonical dyadic intervals in $[T]$, defined using the original global clock, and set
\begin{align*}
\mathcal J_j\triangleq\{J\in\mathcal J:J\text{ starts at or after round }\tau_j\}.
\end{align*}
Thus $\mathcal J_j$ contains every interval scheduled to start from $\tau_j$ through $T$, including those starting after $\rho_j$. In the continuation, every $J\in\mathcal J_j$ receives the rate $\eta_J=\alpha\eta_j/U_J$ and runs for its full interval. For every endpoint $n\in\{\tau_j,\ldots,T\}$, we also define
\begin{align*}
\calD_j(n)\triangleq\left\{J\in\mathcal J_j:
\begin{array}{l}
J\subseteq\{\tau_j,\ldots,n\},\ \mathbf{u}^{(S)}\text{ is constant on }J,\\
J\text{ is maximal under inclusion with these properties}
\end{array}
\right\}.
\end{align*}
Specifically, $\calD_j(n)$ represents the dyadic decomposition of the comparator's constant segments within $\{\tau_j,\ldots,n\}$. Its intervals are disjoint and cover this prefix, with a fixed comparator arm on each interval. The oblivious loss table fixes the comparator before the algorithm randomizes. Hence, for fixed $n$, the only randomness in this cover comes from $\tau_j$, making it $\cF_j$-measurable.

For any dyadic interval $J\in\mathcal J_j$ on which the comparator plays a fixed arm $a_J$, define its \emph{prefix advantage} in the fixed-rate continuation by
\begin{align}
A_J(n)
\triangleq
\sum_{\substack{t\in J,\ t\leq n\\b_t=0}}
\ip{q_t-e_{a_J}}{\ell_t}.
\label{eq:local-advantage}
\end{align}
Here $n\in\{\tau_j,\ldots,T\}$. Here and below, per-round quantities inside a continuation calculation refer to that continuation. To analyze the randomized rates, note that for a fixed endpoint $n$, conditioning on $\cF_j$, all continuation routing bits, and the independent main-action seeds \emph{fixes} the reference path $q_t$ and hence the advantages $A_J(n)$, while leaving the $U_J$ independent and uniform. This again uses the obliviousness of the losses and the fact that the continuation updates $q_t$ only on main rounds and never restarts. At the actual endpoint, the cover $\calD_j(\rho_j)$ and advantages $A_J(\rho_j)$ coincide with those generated by the algorithm on $I_j$.

With all the above defined, our goal is to bound the sum of the advantages over this cover, shown in the following lemma. Since each round is routed to a challenge round with probability $1/2$ independently of the preceding history, the same bound, up to a factor of two, also controls the expected regret of the full main sequence, which is our second guarantee for the $\{q_t\}_{t\in[T]}$ sequence.

\begin{lemma}
\label{lem:dyadic-advantage}
\pref{alg:main} guarantees that every reached epoch $I_j$ satisfies
\begin{align}
\E_j\left[
\sum_{J\in\calD_j(\rho_j)}A_J(\rho_j)
\right]
\leq
\order\left(
\eta_jKTL^2
+
\sqrt{(S+1)KT}\,L^3
\right).
\label{eq:dyadic-advantage}
\end{align}
Consequently, we also have
\begin{align}
\E_j\left[
\Reg_{I_j}^{\mathbf q}(\mathbf{u}^{(S)})
\right]
\leq
\order\left(
\eta_jKTL^2
+
\sqrt{(S+1)KT}\,L^3
\right).
\label{eq:main-sequence-progress}
\end{align}
\end{lemma}

Note that \pref{eq:main-sequence-progress} does not have terms that are inversely dependent on $\eta_j$. Therefore, this regret guarantee is well controlled when $\eta_j$ is small. To prove this lemma, we use several supporting results shown as follows. The first lemma bounds the size of the comparator-dependent dyadic cover. The comparator has at most $S+1$ constant segments, each requiring only $\order(\log T)$ canonical dyadic intervals, so the total cover size is $\order((S+1)\log T)$.

\begin{lemma}
\label{lem:dyadic-cover}
Fix a comparator sequence $\mathbf{u}^{(S)}$ with at most $S$ switches. For every reached epoch $I_j$ of \pref{alg:main} and every $n\in\{\tau_j,\ldots,T\}$, the family $\calD_j(n)$ partitions $\{\tau_j,\ldots,n\}$, the comparator is constant on every $J\in\calD_j(n)$, and $|\calD_j(n)|\leq(S+1)L$.
\end{lemma}

The proof is deferred to \pref{app:dyadic-cover} and follows by counting dyadic intervals within each constant segment. {We next consider the intervals whose randomized rates allow their advantages to exceed the OMD initialization costs.} Set $d_j\triangleq4L/\eta_j$. We use $\delta_0\triangleq T^{-3}$ as the failure probability. Define
\begin{align}\label{eq:selected-local-gain}
W_j(n)&\triangleq
\alpha\sum_{J\in\calD_j(n)}A_J(n)\1\{\eta_JA_J(n)\geq4L\}.
\end{align}

Specifically, $W_j(n)$ is the comparator's total scaled advantage over intervals in $\calD_j(n)$ satisfying $\eta_JA_J(n)\geq4L$. Only positive advantages contribute, so $W_j(n)\geq0$. The next lemma then uses the conditional independence established above to bound $\sum_{J\in\calD_j(n)}A_J(n)$ in terms of $W_j(n)$. We will then control $W_j(n)$ using the credit bound and the restart rule, yielding \pref{eq:dyadic-advantage}.

\begin{lemma}
\label{lem:random-rate-recovery}
For every reached epoch $I_j$ of \pref{alg:main} and every fixed continuation endpoint $n\in\{\tau_j,\ldots,T\}$, conditional on $\cF_j$, with probability at least $1-\delta_0$,
\begin{align}
\sum_{J\in\calD_j(n)}A_J(n)
\leq
\frac{W_j(n)+d_jL/2}{\alpha}
+
\frac{\sqrt{2(S+1)Ld_j\bigl(W_j(n)+d_jL/2\bigr)}}{\alpha}.
\label{eq:advantage-from-selected}
\end{align}
\end{lemma}

The proof is deferred to \pref{app:random-rate-recovery}. To give intuition for this bound, condition on $\cF_j$ and the continuation's routing and main-action randomness. Then the local advantages are fixed while the interval-rate draws remain independent. For $0<\alpha A_J(n)<d_j$, {the probability that $\eta_JA_J(n)\geq4L$ is} $\alpha A_J(n)/d_j$, so its conditional expected contribution is $\alpha^2 A_J(n)^2/d_j$. Advantages with $\alpha A_J(n)\geq d_j$ enter directly, while nonpositive advantages do not increase the sum on the left-hand side. To obtain a high-probability bound, we use these same cases to compare each interval's advantage with its contribution to $W_j(n)$ through an exponential-moment inequality. Since the interval-rate draws remain independent under the conditioning above, we can combine these inequalities to obtain \pref{eq:advantage-from-selected}.

We then relate $W_j(n)$ to the observable credit. The next lemma shows that the credit captures a constant fraction of $W_j(n)$ after subtracting loss-estimation errors and the auxiliary subroutines' regret. For notational convenience, define $\Gamma_j$ to be the following sum over the continuation of epoch $j$:
\begin{align}
\Gamma_j
\triangleq
\sum_{J\in\mathcal J_j}|J|\min\{K\eta_J,\alpha\},
\label{eq:Gamma-j}
\end{align}
where $|J|$ denotes the length of $J$. We use the sum over the whole continuation through $T$ to bound the update and estimation terms for every prefix, so $\Gamma_j$ does not depend on $n$.

\begin{lemma}
\label{lem:selected-credit-recovery}
For every reached epoch $I_j$ of \pref{alg:main} and every fixed continuation endpoint $n\in\{\tau_j,\ldots,T\}$, conditional on $\cF_j$, with probability at least $1-\delta_0$,
\begin{align}
C_{j,n}^{\mathrm{cont}}
\geq\frac34W_j(n)-{\frac{KT\eta_j}{2}}-4\Gamma_j.
\label{eq:credit-controls-selected}
\end{align}
\end{lemma}

This lemma will be used in two ways. In the proof of \pref{lem:dyadic-advantage}, we evaluate \pref{eq:credit-controls-selected} at the actual epoch endpoint $n=\rho_j$. The restart rule and the bound $Z_t\leq1$ give $C_{j,\rho_j}^{\mathrm{cont}}=C_j\leq {4}KT\eta_j+1$, which yields an upper bound on $W_j(\rho_j)$. Since $W_j(n)\geq0$, the same inequality also gives $C_{j,n}^{\mathrm{cont}}\geq-{KT\eta_j/2}-4\Gamma_j$. This lower bound will also be used in Step~3 to control the possible negative credit in the final epoch, which may end without reaching the restart threshold.

To prove \pref{lem:selected-credit-recovery}, we decompose $W_j(n)$ into the auxiliary subroutines' collective improvement over $q_t$ and their total weighted regret, both of which are weighted by $\alpha$ by the definition of $W_j(n)$. For each interval $J\in\calD_j(n)$ satisfying $\eta_JA_J(n)\geq4L$, we compare its auxiliary subroutine with $e_{a_J}$. We compare every other subroutine in the continuation with $\mathbf0$. Under the completed-vector representation, $\mathbf0$ corresponds to following $q_t$. This is where the non-uniform prior keeps the total initialization penalty for comparisons with $q_t$ small. However, different from the classic analysis, we point out one caveat in analyzing the OMD subroutines: the intervals satisfying $\eta_JA_J(n)\geq4L$ depend on the randomized rates, so we must justify the expectation calculations for this random choice of intervals. Fortunately, the separation of main updates from auxiliary tests ensures that, conditional on the continuation's routing bits, main-round randomness, and interval rates, $i_t$ is still sampled from the current $p_t$ on each challenge round, given the earlier challenge-round draws. We can therefore bound the differences between estimated and actual losses and relate the OMD comparison to the observable credit. Then, for a comparison with $e_{a_J}$, implicit exploration controls the comparator arm's estimation error, and we can show that this error together with the initialization penalty is at most $\alpha L/\eta_J$ with high probability. The condition $\eta_JA_J(n)\geq4L$ bounds this by $\alpha A_J(n)/4$, so their total contribution is at most $W_j(n)/4$. It remains to bound the errors in estimating the main learner's losses and relating the estimates to credit, together with the OMD update terms and the initialization penalties for following $q_t$. These terms are bounded by $KT\eta_j/2+4\Gamma_j$ with high probability. Combining these bounds gives \pref{eq:credit-controls-selected}. We defer the technical details of the proof to \pref{app:selected-credit-recovery}.

It remains to control $\Gamma_j$. Although $\eta_J$ has infinite expectation, each contribution to $\Gamma_j$ is truncated, which gives the finite expectation below.

\begin{lemma}
\label{lem:continuation-charge}
For every reached epoch $I_j$, \pref{alg:main} guarantees that $ \E_j[\Gamma_j]\leq{\frac{KT\eta_j}{16}}. $
\end{lemma}

The proof is via a direct calculation and is deferred to \pref{app:continuation-charge}. We now combine the preceding three lemmas to bound $\sum_{J\in\calD_j(\rho_j)}A_J(\rho_j)$.

\begin{proof}[Proof of \pref{lem:dyadic-advantage}]
Fix a reached epoch and condition on $\cF_j$. Let $\mathcal E_j$ be the event that both \pref{eq:advantage-from-selected} and \pref{eq:credit-controls-selected} hold for every $n\in\{\tau_j,\ldots,T\}$. Applying \pref{lem:random-rate-recovery} and \pref{lem:selected-credit-recovery} at each fixed endpoint and taking a union bound gives $\Pp(\mathcal E_j\mid\cF_j)\geq1-2T\delta_0$. On this event, we may evaluate both estimates at $n=\rho_j$ even though $\rho_j$ is random.

Since $C_{j,\rho_j}^{\mathrm{cont}}=C_j$, the restart rule and $|Z_t|\leq1$ give $C_j\leq {4}KT\eta_j+1$. Rearranging \pref{eq:credit-controls-selected} and substituting this upper bound on $C_j$ gives, on $\mathcal E_j$,
\begin{align}
W_j(\rho_j)
&\leq
\frac43\left(C_j+{\frac{KT\eta_j}{2}}+4\Gamma_j\right)\leq\frac43\left[{\frac92}KT\eta_j+1+4\Gamma_j\right]=
\order(KT\eta_j+\Gamma_j).
\label{eq:selected-gain-upper}
\end{align}
Here the last equality uses $KT\eta_j\geq1$. Since $\eta_j\geq\eta_1$, we have $d_jL=4L^2/\eta_j{\leq4KT\eta_j}$. Together with \pref{eq:selected-gain-upper}, this gives $W_j(\rho_j)+d_jL/2=\order(KT\eta_j+\Gamma_j)$ on $\mathcal E_j$. Substituting into \pref{eq:advantage-from-selected} therefore yields
\begin{align*}
\sum_{J\in\calD_j(\rho_j)}A_J(\rho_j)
\leq
\order\left(
\frac{KT\eta_j+\Gamma_j+\sqrt{(S+1)Ld_j(KT\eta_j+\Gamma_j)}}{\alpha}
\right).
\end{align*}

On $\mathcal E_j^c$, $\sum_{J\in\calD_j(\rho_j)}A_J(\rho_j)\leq T$, and $\Pp(\mathcal E_j^c\mid\cF_j)\leq2T\delta_0$. Since the preceding upper bound is nonnegative, taking conditional expectations and then applying Jensen's inequality gives
\begin{align*}
\E_j\left[\sum_{J\in\calD_j(\rho_j)}A_J(\rho_j)\right]
&\leq
\order\left(
\frac{KT\eta_j+\E_j[\Gamma_j]
+\sqrt{(S+1)Ld_j}\,\E_j[\sqrt{KT\eta_j+\Gamma_j}]}{\alpha}
+2T^2\delta_0
\right)\\
&\leq
\order\left(
\frac{KT\eta_j+\E_j[\Gamma_j]
+\sqrt{(S+1)Ld_j(KT\eta_j+\E_j[\Gamma_j])}}{\alpha}
+1
\right)\\
&\leq
\order\left(
\frac{KT\eta_j+\sqrt{(S+1)Ld_jKT\eta_j}}{\alpha}+1
\right).
\end{align*}
Here we used $2T^2\delta_0=2/T\leq1$ and, in the last inequality, $\E_j[\Gamma_j]\leq {KT\eta_j/16}$ from \pref{lem:continuation-charge}. Since $1/\alpha={L^2}$ and $(S+1)Ld_jKT\eta_j=4(S+1)KTL^2$, this proves \pref{eq:dyadic-advantage}.

To obtain \pref{eq:main-sequence-progress}, set $f_t\triangleq\ip{q_t-u_t^{(S)}}{\ell_t}$. The fixed comparator and losses make $f_t$ $\cH_t$-measurable. For every bounded $\cF_j$-measurable random variable $X$, the stopped-history argument in \pref{app:fixed-share} makes $X\1\{t\in I_j\}f_t$ $\cH_t$-measurable. Applying $\Pp(b_t=0\mid\cH_t)=1/2$ to these products and summing over deterministic rounds gives
\begin{align*}
\E_j\left[
\Reg_{I_j}^{\mathbf q}(\mathbf{u}^{(S)})
\right]
=
2\E_j\left[
\sum_{t\in I_j:b_t=0}f_t
\right]=
2\E_j\left[
\sum_{J\in\calD_j(\rho_j)}A_J(\rho_j)
\right].
\end{align*}
Combining this identity with \pref{eq:dyadic-advantage} proves \pref{eq:main-sequence-progress}.
\end{proof}

\paragraph{Step 3: use the credit to control the stability terms.}
Every completed epoch reaches the restart threshold and therefore has credit at least ${4}KT\eta_j$. Only the final epoch can end before reaching this threshold, so its credit may be negative. As in \pref{eq:proof-overview-decomposition}, this increases regret through $-C_N$. To control this possible negative credit at its random endpoint, define $F_j \triangleq \max_{\tau_j\leq n\leq T} \left(-C_{j,n}^{\mathrm{cont}}\right)_+.$ This is the largest magnitude of negative credit generated by the fixed-rate continuation over all possible endpoints. Since the actual epoch and its continuation coincide through $\rho_j$, its credit is at least $-F_j$. Moreover, $W_j(n)\geq0$, so dropping this nonnegative term from \pref{eq:credit-controls-selected} gives a lower bound on the credit at every fixed endpoint. Combining these endpoint-wise bounds with \pref{lem:continuation-charge} gives the following estimate.

\begin{lemma}
\label{lem:maximum-debt}
Every reached epoch $I_j$ of \pref{alg:main} satisfies $ \E_j[F_j] \leq {KT\eta_j}. $
\end{lemma}

\begin{proof}
For each fixed $n\in\{\tau_j,\ldots,T\}$, \pref{eq:credit-controls-selected} and $W_j(n)\geq0$ imply $C_{j,n}^{\mathrm{cont}}\geq-{KT\eta_j/2}-4\Gamma_j$ except on an event of conditional probability at most $\delta_0$. A union bound over the possible endpoints therefore gives $F_j\leq {KT\eta_j/2}+4\Gamma_j$ with conditional probability at least $1-T\delta_0$. On the failure event, the deterministic bound $F_j\leq T$ follows because every continuation credit increment has absolute value at most one. Taking conditional expectations and applying \pref{lem:continuation-charge} yields
\begin{align*}
\E_j[F_j]\leq
{\frac{KT\eta_j}{2}}
+4\E_j[\Gamma_j]
+T^2\delta_0\leq
{\frac{KT\eta_j}{2}}
+{\frac{KT\eta_j}{4}}
+\frac1T
\leq
{KT\eta_j}.
\end{align*}
The choices $\delta_0=T^{-3}$ and $\eta_j\geq\eta_1={L}/\sqrt{KT}$ ensure $T^2\delta_0=1/T{\leq1\leq KT\eta_j/4}$.
\end{proof}

We next compare the total credit with the main learner's stability cost. Each completed epoch contributes at least ${4}KT\eta_j$ in credit and incurs at most $KT\eta_j$ in stability. The final epoch may instead have negative credit, but \pref{lem:maximum-debt} controls this negative credit in expectation. Since the learning rates double across epochs, the credit accumulated before the final epoch controls both the stability costs and the expected negative credit in the final epoch. The following lemma makes this credit comparison precise.

\begin{lemma}
\label{lem:credit-pays-stability}
\pref{alg:main} guarantees that the epoch credits satisfy $ \E\left[ \sum_{j=1}^N\eta_jKH_j - \sum_{j=1}^NC_j \right] \leq {2}KT\eta_1. $
\end{lemma}

\begin{proof}
Set $F_j\triangleq0$ for unreached epochs. Since $\{j\leq N\}\in\cF_j$, the tower property and \pref{lem:maximum-debt} give
\begin{align}
\E\left[\sum_{j=1}^N F_j\right]
&=\sum_{j=1}^T
  \E\left[\1\{j\leq N\}\E_j[F_j]\right]\leq{KT}\sum_{j=1}^T
  \eta_j\Pp(j\leq N)
={KT}\E\left[\sum_{j=1}^N\eta_j\right].
\label{eq:sum-negative-credit}
\end{align}
Every completed epoch $j<N$ satisfies $C_j\geq {4}KT\eta_j$, whereas the final epoch satisfies $C_N\geq-F_N$. Since every $F_j$ is nonnegative, we can lower bound $\sum_{j=1}^NC_j$ as follows:
\begin{align}
\sum_{j=1}^NC_j
\geq
{4}KT\sum_{j=1}^{N-1}\eta_j
-
\sum_{j=1}^NF_j.
\label{eq:credit-completed-epochs}
\end{align}
Taking expectations, applying \pref{eq:sum-negative-credit}, together with $\sum_{j=1}^N\eta_j=\eta_1+2\sum_{j=1}^{N-1}\eta_j$ and $H_j\leq T$ give that
\begin{align*}
\E\left[
\sum_{j=1}^N\eta_jKH_j-\sum_{j=1}^NC_j
\right]&\leq \E\left[
\sum_{j=1}^N\eta_jKT-{4}KT\sum_{j=1}^{N-1}\eta_j+\sum_{j=1}^NF_j
\right]\\
&\leq2KT\E\left[\sum_{j=1}^{N}\eta_j\right]-4KT\E\left[\sum_{j=1}^{N-1}\eta_j\right]\\
&=2KT\eta_1,
\end{align*}
which finishes the proof.
\end{proof}

\paragraph{Step 4: combine the two epoch bounds.}
Finally, we apply the two main-sequence bounds in complementary rate regimes. For any fixed value of $S$, let $j_S$ be the largest deterministic index satisfying $\eta_{j_S}\leq\eta_1\sqrt{S+1}$. Because $\eta_j=2^{j-1}\eta_1$, there are at most $L$ such indices and $ \sum_{j=1}^{j_S}\eta_j \leq 2\eta_1\sqrt{S+1}.$ Applying \pref{lem:dyadic-advantage} to the reached epochs with $j\leq j_S$ and summing the terms proportional to $\eta_j$ geometrically and the remaining term over at most $L$ epochs gives
\begin{align}
\E\left[
\sum_{j\leq j_S}
\Reg_{I_j}^{\mathbf q}(\mathbf{u}^{(S)})
\right]
\leq
\order\left(
\eta_1KT\sqrt{S+1}\,L^2
+
\sqrt{(S+1)KT}\,L^4
\right).
\label{eq:small-rate-epochs}
\end{align}
For the reached epochs with $j>j_S$, apply \pref{lem:fixed-share}. Since $\eta_{j_S+1}>\eta_1\sqrt{S+1}$ and the learning rates double, the inverse-rate terms form a geometric sum controlled by their first term. Therefore,
\begin{align}
\E\left[
\sum_{j>j_S}
\Reg_{I_j}^{\mathbf q}(\mathbf{u}^{(S)})
\right]
\leq
\order\left(
\frac{\sqrt{S+1}\log T}{\eta_1}
\right)
+
\E\left[
\sum_{j>j_S}\eta_jKH_j
\right].
\label{eq:large-rate-epochs}
\end{align}

Substituting \pref{eq:small-rate-epochs} and \pref{eq:large-rate-epochs} into \pref{eq:main-regret-decomposition} gives
\begin{align*}
\E\left[\Reg_T(\mathbf{u}^{(S)})\right]
&\leq
\order\left(
\eta_1KT\sqrt{S+1}\,L^2
+
\sqrt{(S+1)KT}\,L^4
+
\frac{\sqrt{S+1}\log T}{\eta_1}
\right)
\\
&+
\E\left[
\sum_{j>j_S}\eta_jKH_j
-
\sum_{j=1}^NC_j
\right].
\end{align*}
Since each stability term $\eta_jKH_j$ is nonnegative, \pref{lem:credit-pays-stability} gives that
\begin{align*}
\E\left[
\sum_{j>j_S}\eta_jKH_j
-
\sum_{j=1}^NC_j
\right]
&\leq
\E\left[
\sum_{j=1}^N\eta_jKH_j
-
\sum_{j=1}^NC_j
\right]\leq {2}KT\eta_1.
\end{align*}
Combining these inequalities yields
\begin{align}
\E\left[\Reg_T(\mathbf{u}^{(S)})\right]
\leq
\order\left(
\eta_1KT\sqrt{S+1}\,L^2
+
\sqrt{(S+1)KT}\,L^4
+
\frac{\sqrt{S+1}\log T}{\eta_1}
+
KT\eta_1
\right).
\label{eq:final-bound-before-substitution}
\end{align}
Finally, substituting $\eta_1={L}/\sqrt{KT}$ and $L=\lceil{8\log T}\rceil$ into \pref{eq:final-bound-before-substitution} finishes the proof.

%% file: sections/07_conclusion.tex
\section{Conclusion}
\label{sec:conclusion}

We study switching regret for adversarial multi-armed bandits when the number of comparator switches is unknown. Against an oblivious adversary, a single algorithm achieves $\order(\sqrt{(S+1)KT}\log^4 T)$ expected regret for every number of switches $S$. The algorithm is based on EXP3.S and adapts by accumulating observable credit from the interval subroutines and doubling the main learning rate when this credit reaches a prescribed threshold. Future work includes further reducing the logarithmic factors and extending the construction to broader bandit problems such as linear and convex bandits.

%% file: sections/a_proofs.tex
\begingroup
\section{Omitted Details in \pref{sec:proof}}
\label{app:proof}

This appendix supplies the proofs deferred from \pref{sec:proof}. \pref{app:continuation-conditioning} records the histories and conditional independence used throughout the analysis. \pref{app:fixed-share} proves the bound for the main learner, and \pref{app:dyadic-cover} constructs the comparator's dyadic cover. \pref{app:random-rate-recovery} recovers the full advantage from the intervals $J\in\calD_j(n)$ satisfying $\eta_JA_J(n)\geq4L$, \pref{app:selected-credit-recovery} relates their advantage to credit, and \pref{app:continuation-charge} bounds the expected continuation cost. The OMD and estimation arguments used in \pref{app:selected-credit-recovery} are proved in \pref{app:auxiliary-estimates}.

\subsection{Histories and Conditioning}
\label{app:continuation-conditioning}

 We use $\cF_j$, $\cH_t$, and $\cG_t$ with their meanings in \pref{sec:proof}. Let $\cH_t^{\mathrm{end}}$ denote the history of observations and random variables revealed through the update and restart decision at round $t$, with $\cH_0^{\mathrm{end}}$ containing only the fixed input. Then
\begin{align*}
\cH_t
&=\sigma\bigl(\cH_{t-1}^{\mathrm{end}},
                  (U_J)_{J\text{ launched at }t}\bigr),
&\cG_t&=\sigma(\cH_t,b_t),\\
\cH_{t-1}^{\mathrm{end}}
&\subseteq\cH_t\subseteq\cG_t\subseteq\cH_t^{\mathrm{end}}.
\end{align*}
In particular, the epoch-start history is $\cF_j=\cH_{\tau_j-1}^{\mathrm{end}}$, before the interval launches at $\tau_j$. Here the right-hand side denotes the stopped sigma-algebra, since $\tau_j-1$ is a stopping time for $(\cH_t^{\mathrm{end}})_{t=0}^T$. For an unreached epoch, set $\tau_j=T+1$. Then $\{j\leq N\}\in\cF_j$. For each fixed endpoint $n\geq\tau_j$, the cover $\calD_j(n)$ is $\cF_j$-measurable, since it is determined by $\tau_j$, $n$, and the fixed comparator sequence. Moreover, $\1\{t\in I_j\}$ is $\cH_t$-measurable, because the restart decision is made after round $t$ and the threshold-crossing round remains in $I_j$.

To couple each epoch with its fixed-rate continuation, assign to every deterministic epoch index $j$ a collection of random variables: routing bits $b_{j,t}\sim\Bern(1/2)$, interval-rate variables $U_{j,J}^{\mathrm{rate}}\sim\operatorname{Unif}(0,1)$, and action-sampling variables $U_{j,t}^{\mathrm{main}},U_{j,t}^{\mathrm{chal}} \sim\operatorname{Unif}(0,1)$. These variables are mutually independent across all epochs, rounds, intervals, and types. The routing bit determines whether round $t$ uses the main or challenge branch. The variable $U_{j,J}^{\mathrm{rate}}$ determines the learning rate assigned to interval $J$. The variables $U_{j,t}^{\mathrm{main}}$ and $U_{j,t}^{\mathrm{chal}}$ are used to sample the arm on the main and challenge branches, respectively. Concretely, a uniform variable $u$ samples from a distribution $r\in\Delta_K$ by choosing
\[
i=\min\left\{a\in[K]:u\leq\sum_{k=1}^a r_k\right\}.
\]
Thus the main branch uses $r=q_t$ and $u=U_{j,t}^{\mathrm{main}}$, whereas the challenge branch uses $r=p_t$ and $u=U_{j,t}^{\mathrm{chal}}$. Independence concerns these underlying random variables. The sampled arms may depend on earlier observations through their sampling distributions.

The actual algorithm reveals only the random variables needed for its current launches and chosen branch. Before epoch $j$ starts, none of its assigned variables has been used, so conditional on $\cF_j$ they remain independent and have their specified distributions. The fixed-rate continuation from \pref{sec:proof} uses the same variables as the actual epoch but suppresses subsequent restarts. Within this continuation, denote by $U_J$ the variable $U_{j,J}^{\mathrm{rate}}$ and set $\eta_J=\alpha\eta_j/U_J$ for $J\in\mathcal J_j$. Hence $U_J$ is exactly the uniform variable used at launch in \pref{alg:main}, with the epoch index omitted.

The actual epoch and its continuation consequently have identical states, actions, and credit through $\rho_j$, including the threshold-crossing round. Afterward, actual play uses the random variables assigned to subsequent epochs, while the continuation continues using those assigned to epoch $j$. Therefore $C_{j,\rho_j}^{\mathrm{cont}}=C_j$. As in \pref{sec:proof}, per-round symbols in continuation calculations refer to this continued process. In particular, rates assigned after $\rho_j$ need not equal those assigned at the corresponding actual future launches.

For the continuation analysis, we reveal information in three stages. The following conditioning sigma-algebras include future private random variables and are used only in the analysis. First, $\cF_j^{\mathrm{path}}$ includes all continuation routing bits and main-branch action-sampling variables. The main learner's distribution $q_t$ is updated only on main rounds, where $p_t=q_t$, and is frozen on challenge rounds. Since the losses are fixed, the revealed variables determine the entire sequence $(q_t)_{t=\tau_j}^T$. For each fixed endpoint $n\in\{\tau_j,\ldots,T\}$, it also determines the cover $\calD_j(n)$ and the advantages $(A_J(n))_{J\in\calD_j(n)}$, while leaving the interval-rate variables independent and uniform. Second, $\cF_j^{\mathrm{rate}}$ additionally reveals all continuation interval rates, thereby determining $\Gamma_j$ and, for each fixed endpoint $n$, the intervals $J\in\calD_j(n)$ satisfying $\eta_JA_J(n)\geq4L$. Finally, $\cV_{j,t}^{\mathrm{chal}}$ additionally reveals the challenge-branch action-sampling variables before round $t$. This determines the current interval states and $p_t$, while leaving $U_{j,t}^{\mathrm{chal}}$ independent and uniform. Formally,
\begin{align*}
\cF_j^{\mathrm{path}}
&=\sigma\bigl(\cF_j,
             (b_{j,t},U_{j,t}^{\mathrm{main}})_{t=\tau_j}^T\bigr),\\
\cF_j^{\mathrm{rate}}
&=\sigma\bigl(\cF_j^{\mathrm{path}},
                    (U_J)_{J\in\mathcal J_j}\bigr),\\
\cV_{j,t}^{\mathrm{chal}}
&=\sigma\bigl(\cF_j^{\mathrm{rate}},
                    (U_{j,s}^{\mathrm{chal}})_{\tau_j\leq s<t}\bigr).
\end{align*}
These histories satisfy
\[
\cF_j\subseteq\cF_j^{\mathrm{path}}
\subseteq\cF_j^{\mathrm{rate}}
=\cV_{j,\tau_j}^{\mathrm{chal}}
\subseteq\cV_{j,t}^{\mathrm{chal}}
\subseteq\cV_{j,t+1}^{\mathrm{chal}}.
\]
Table~\ref{tab:filtrations} summarizes the information determined by each history and the random variables used in the corresponding conditional calculation.

\begin{table}[!htbp]
\centering
\renewcommand{\arraystretch}{1.16}
\begin{tabular}{@{}>{\raggedright\arraybackslash}p{0.20\linewidth}>{\raggedright\arraybackslash}p{0.39\linewidth}>{\raggedright\arraybackslash}p{0.35\linewidth}@{}}
\hline
History & Information already determined & Remaining randomness used in the analysis \\
\hline
$\cH_t^{\mathrm{end}}$ & Actual play through round $t$, including its restart decision & Future launch, routing, and action draws \\
$\cF_j$ & History before epoch $j$ starts, $\tau_j$ and the reset state & All random variables assigned to epoch $j$ \\
$\cH_t$ & Past play, current interval rates and states, and $q_t$ & Routing bit $b_t$, followed by the branch action \\
$\cG_t$ & $\cH_t$ plus $b_t$ and $p_t$ & Arm $i_t\sim p_t$ \\
\hline
$\cF_j^{\mathrm{path}}$ & Continuation routing bits and the main learner's sequence $(q_t)_{t=\tau_j}^T$, $\calD_j(n)$ and $(A_J(n))_{J\in\calD_j(n)}$ for fixed $n$ & Interval-rate and challenge-branch action-sampling variables \\
$\cF_j^{\mathrm{rate}}$ & $\cF_j^{\mathrm{path}}$ plus all continuation rates, the intervals $J\in\calD_j(n)$ satisfying $\eta_JA_J(n)\geq4L$ for fixed $n$, and $\Gamma_j$ & Challenge-branch action-sampling variables \\
$\cV_{j,t}^{\mathrm{chal}}$ & $\cF_j^{\mathrm{rate}}$ plus challenge-branch variables before $t$, current $p_t$ and interval states & The independent uniform variable $U_{j,t}^{\mathrm{chal}}$ \\
\hline
\end{tabular}
\caption{Actual histories and the enlarged histories used to analyze a fixed-rate continuation. The lower three rows describe the successive conditioning steps in the proof.}
\label{tab:filtrations}
\end{table}

In particular, the independence of $U_{j,t}^{\mathrm{chal}}$ from $\cV_{j,t}^{\mathrm{chal}}$ gives, on a challenge round,
\[
\Pp(i_t=i\mid\cV_{j,t}^{\mathrm{chal}})=p_{t,i}.
\]
We use this conditional law for challenge-round moments. For routing unbiasedness, we instead condition on $\cH_t$, before $b_t$ is revealed. The tower property transfers bounds obtained under the enlarged histories back to $\cF_j$. For each fixed $n$, the set $\{J\in\calD_j(n):\eta_JA_J(n)\geq4L\}$ is $\cF_j^{\mathrm{rate}}$-measurable. At the random endpoint $\rho_j$, however, the set $\{J\in\calD_j(\rho_j):\eta_JA_J(\rho_j)\geq4L\}$ need not be. We therefore take a union over fixed endpoints before substituting $\rho_j$, as in the proof of \pref{lem:dyadic-advantage}.

Finally, we use the completed point $\bar x$ and completed relative entropy from \pref{sec:prelim}. Each interval is initialized with $x_{t,J,i}=1/(K(T+1))$ for $i\in[K]$. We denote the corresponding initial distribution on the completed simplex by
\[
\pi\triangleq
\left(\frac{T}{T+1},
      \frac{1}{K(T+1)},\ldots,\frac{1}{K(T+1)}\right)
\in\Delta_{K+1}.
\]
Its coordinate $0$ is the initial reference mass $1-\sum_{i=1}^K x_{t,J,i}=T/(T+1)$, and its remaining coordinates are the initial arm masses.

\subsection{Proof of \pref{lem:fixed-share}}
\label{app:fixed-share}

We prove \pref{lem:fixed-share} by adapting the standard potential analysis of EXP3.S~\citep{auer2002nonstochastic} to the  random routing bit $b_t$. We first establish a pathwise bound for every epoch prefix, keeping the update before a possible reset separate from the reset itself. We then use predictable epoch membership to take expectations at the random endpoint.

\begin{proof}[Proof of \pref{lem:fixed-share}]
Fix a reached epoch $I_j=\{\tau_j,\ldots,\rho_j\}$ and condition on $\cF_j$. The epoch starts with $q_{\tau_j,i}=1/K$ and uses the fixed learning rate $\eta_j$. Recall that on main rounds, the entropy-OMD iterate in \pref{alg:main} satisfies
\begin{align*}
\widetilde q_{t+1}
&=\arg\min_{q\in\Delta_K}
\left\{\ip{q}{\widehat\ell_t^q}+\frac{1}{\eta_j}\KL(q\|q_t)\right\}.
\end{align*}
For the analysis, use $\widehat\ell_t^q=0$ and $\widetilde q_{t+1}=q_t$ on challenge rounds. To distinguish the within-epoch update from a possible reset, define for every $t\in I_j$
\begin{align*}
q_{t+1}^{+}
&\triangleq
\begin{cases}
(1-1/T)\widetilde q_{t+1}+\frac{\mathbf1}{KT},&b_t=1,\\
q_t,&b_t=0.
\end{cases}
\end{align*}
Thus $q_{t+1}=q_{t+1}^{+}$ for $\tau_j\leq t<\rho_j$. At the final round $\rho_j$, a reset may replace $q_{\rho_j+1}^{+}$ by the uniform distribution for the next epoch. The potential calculation below uses $q_{\rho_j+1}^{+}$, so this reset does not enter the current epoch's bound.

Conditional on $\cH_t$, the routing bit $b_t$ is fair and a main round samples from $p_t=q_t$, so $\Pp(b_t=1,i_t=i\mid\cH_t)=q_{t,i}/2$. Consequently,
\begin{align}
\E\left[\widehat\ell^q_{t,i}\mid\cH_t\right]
&=\frac{q_{t,i}}{2}\frac{2\ell_{t,i}}{q_{t,i}}
=\ell_{t,i},
\label{eq:main-unbiased}\\
\E\left[\ip{q_t}{(\widehat\ell_t^q)^2}\mid\cH_t\right]
&=\sum_{i=1}^K\frac{q_{t,i}}{2}\,q_{t,i}
  \frac{4\ell_{t,i}^2}{q_{t,i}^2}
=2\sum_{i=1}^K\ell_{t,i}^2
\leq2K.
\label{eq:main-second-moment}
\end{align}

On a main round, the entropy-OMD update in \pref{alg:main} gives that $\widetilde q_{t+1,i} =\frac{q_{t,i}e^{-\eta_j\widehat\ell^q_{t,i}}}{\mathsf Z_t^q}$ with $\mathsf Z_t^q\triangleq \sum_{k=1}^Kq_{t,k}e^{-\eta_j\widehat\ell^q_{t,k}}$, and $q_{t+1,i}^{+} =(1-1/T)\widetilde q_{t+1,i}+\frac{1}{KT}$. On challenge rounds, we have $\widehat\ell_t^q=0$ and $\widetilde q_{t+1}=q_{t+1}^{+}=q_t$. According to the definition of $\mathsf Z_t^q$, we have $\mathsf Z_t^q=1$ when $t$ is a challenge round. Then for every round, using $e^{-z}\leq1-z+z^2/2$ for $z\geq0$ and $\log z\leq z-1$ for $z>0$ gives
\[
\log\mathsf Z_t^q
\leq\mathsf Z_t^q-1
\leq-\eta_j\ip{q_t}{\widehat\ell_t^q}
+\frac{\eta_j^2}{2}\ip{q_t}{(\widehat\ell_t^q)^2}.
\]
Rearranging the terms gives that for every $i\in[K]$,
\begin{align}
\eta_j\left(
\ip{q_t}{\widehat\ell_t^q}-\widehat\ell^q_{t,i}
\right)
&=\log\frac{\widetilde q_{t+1,i}}{q_{t,i}}
+\log\mathsf Z_t^q+\eta_j\ip{q_t}{\widehat\ell_t^q}
\notag\\
&\leq\log\frac{\widetilde q_{t+1,i}}{q_{t,i}}
+\frac{\eta_j^2}{2}\ip{q_t}{(\widehat\ell_t^q)^2}
\notag\\
&\leq\log\frac{q_{t+1,i}^{+}}{q_{t,i}}
+\1\{b_t=1\}\log\frac{T}{T-1}
+\frac{\eta_j^2}{2}\ip{q_t}{(\widehat\ell_t^q)^2}.
\label{eq:one-step-fixed-share}
\end{align}
The last step uses $q_{t+1,i}^{+}\geq(1-1/T)\widetilde q_{t+1,i}$ on main rounds. To further analyze the right hand side of \pref{eq:one-step-fixed-share}, for any prefix $\{\tau_j,\ldots,n\}\subseteq I_j$, denote by $a_t$ the comparator arm, so $u_t^{(S)}=e_{a_t}$. Since there is no reset inside this prefix, we have $q_t=q_t^{+}$ for $\tau_j<t\leq n$. Hence, we have
\begin{align*}
\sum_{t=\tau_j}^n
\log\frac{q_{t+1,a_t}^{+}}{q_{t,a_t}}
=\log q_{n+1,a_n}^{+}-\log q_{\tau_j,a_{\tau_j}}
+\sum_{t=\tau_j+1}^n
  \log\frac{q_{t,a_{t-1}}}{q_{t,a_t}}\leq(2S+1)\log T,
\end{align*}
where we use $q_{n+1,a_n}^{+}\leq1$ and $q_{\tau_j,a_{\tau_j}}=1/K$. Moreover, only comparator switches contribute to the last sum, so at most $S$ terms are nonzero. Each term is at most $2\log T$ because every coordinate of $q_t$ is at least $1/(KT)$, a lower bound preserved by uniform mixing on main rounds and by the unchanged reference distribution on challenge rounds. In addition, we have \( \sum_{t=\tau_j}^n\1\{b_t=1\}\log\frac{T}{T-1} \leq T\log\frac{T}{T-1} \leq2. \) Summing \pref{eq:one-step-fixed-share} with $i=a_t$ therefore gives the pathwise prefix bound
\begin{align}
\sum_{t=\tau_j}^n
\ip{q_t-u_t^{(S)}}{\widehat\ell_t^q}
\leq
\frac{(2S+1)\log T+2}{\eta_j}
+\frac{\eta_j}{2}
\sum_{t=\tau_j}^n\ip{q_t}{(\widehat\ell_t^q)^2}.
\label{eq:estimated-fixed-share}
\end{align}
The numerator of the first term is at most $3(S+1)\log T$. Since this bound holds for every realized epoch prefix, it also holds at $n=\rho_j$.

It remains to take expectations without conditioning on the realized endpoint. For each deterministic round $t$, the indicator $\1\{t\in I_j\}$ is $\cH_t$-measurable: whether round $t$ belongs to epoch $j$ is decided before the current routing bit and action. Moreover, whenever $t\in I_j$, the epoch-start history $\cF_j$ has already been revealed. Thus, for any bounded $\cF_j$-measurable random variable $X$, the product $X\1\{t\in I_j\}$ is $\cH_t$-measurable. The preceding measurability property allows us to condition first on $\cH_t$ and then on $\cF_j$ for the terms restricted to $I_j$. Applying \pref{eq:main-unbiased} and \pref{eq:main-second-moment} gives
\begin{align}
\E_j\left[
\sum_{t=1}^T\1\{t\in I_j\}
\ip{q_t-u_t^{(S)}}{\widehat\ell_t^q-\ell_t}
\right]&=\sum_{t=1}^T\E\left[
\E\left[
\1\{t\in I_j\}\ip{q_t-u_t^{(S)}}{\widehat\ell_t^q-\ell_t}
\mid\cH_t
\right]\mid\cF_j
\right]
\notag\\*
&=\sum_{t=1}^T\E_j\left[
\1\{t\in I_j\}
\ip{q_t-u_t^{(S)}}{\E[\widehat\ell_t^q\mid\cH_t]-\ell_t}
\right]
=0,\label{eq:T-epoch-zero}\\
\E_j\left[
\sum_{t=1}^T\1\{t\in I_j\}
\ip{q_t}{(\widehat\ell_t^q)^2}
\right]&=\sum_{t=1}^T\E\left[
\E\left[
\1\{t\in I_j\}\ip{q_t}{(\widehat\ell_t^q)^2}
\mid\cH_t
\right]\mid\cF_j
\right]
\notag\\*
&=\sum_{t=1}^T\E_j\left[
\1\{t\in I_j\}
\E\left[\ip{q_t}{(\widehat\ell_t^q)^2}\mid\cH_t\right]
\right]
\notag\\*
&\leq2K\sum_{t=1}^T\E_j[\1\{t\in I_j\}]
=2K\E_j[H_j].\label{eq:T-epoch-stab}
\end{align}
The inner expectations use the $\cH_t$-measurability of $q_t$, $u_t^{(S)}$, and $\1\{t\in I_j\}$. Setting $n=\rho_j$ in \pref{eq:estimated-fixed-share} and taking $\E_j$, we have
\begin{align*}
\E_j\left[
\sum_{t\in I_j}\ip{q_t-u_t^{(S)}}{\ell_t}
\right]
&=\E_j\left[
\sum_{t\in I_j}\ip{q_t-u_t^{(S)}}{\widehat\ell_t^q}
\right]\tag{since \pref{eq:T-epoch-zero}}\\
&\leq
\frac{(2S+1)\log T+2}{\eta_j}
+\frac{\eta_j}{2}\E_j\left[
\sum_{t\in I_j}\ip{q_t}{(\widehat\ell_t^q)^2}
\right]\tag{since \pref{eq:estimated-fixed-share}}\\
&\leq
\frac{3(S+1)\log T}{\eta_j}
+\eta_jK\E_j[H_j]. \tag{since \pref{eq:T-epoch-stab}}
\end{align*}
This finishes the proof.
\end{proof}

\subsection{Proof of \pref{lem:dyadic-cover}}
\label{app:dyadic-cover}

We now prove \pref{lem:dyadic-cover}, showing that each epoch prefix can be partitioned into at most $(S+1)L$ canonical dyadic intervals, on each of which the comparator is constant. We construct this partition by taking the maximal canonical dyadic intervals contained in each constant comparator segment. At most two such intervals occur at each scale, giving the required logarithmic bound on the cover size.

\begin{proof}[Proof of \pref{lem:dyadic-cover}]
Take all canonical dyadic intervals contained in an integer interval $I\subseteq[T]$ that are maximal under inclusion. Singletons are canonical, so these intervals cover $I$. Canonical dyadic intervals are either disjoint or nested, and maximality therefore makes the chosen intervals disjoint. At any fixed dyadic length, there can be at most two maximal intervals. The dyadic parent of a maximal interval, taken in the full dyadic grid, is not contained in $I$, so it crosses an endpoint of $I$. At a fixed scale, at most one parent crosses each endpoint, and each such parent has at most one child contained in $I$. Otherwise the parent itself would be contained. This gives at most two maximal children. As there are $\lfloor\log_2 T\rfloor+1$ possible dyadic lengths and at most two maximal intervals occur at each length, the construction uses at most $2(\lfloor\log_2 T\rfloor+1)\leq5\log T\leq L$ intervals. Figure~\ref{fig:dyadic-cover} illustrates this construction.

\begin{figure}[t]
\centering
\includegraphics[width=\linewidth]{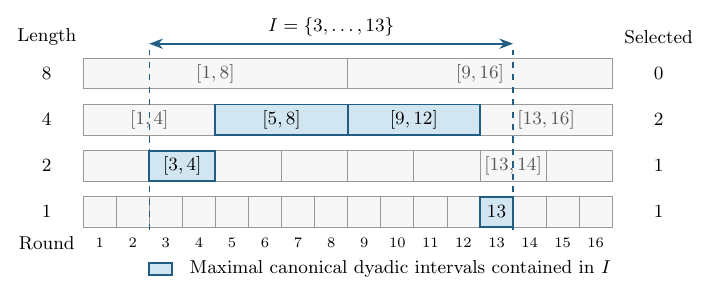}
\caption{Maximal canonical dyadic intervals contained in $I=\{3,\ldots,13\}$, with $T=16$. The highlighted intervals form the disjoint partition $[3,4]\cup[5,8]\cup[9,12]\cup\{13\}$, where brackets denote integer intervals. Each highlighted interval's parent crosses a dashed boundary of $I$, and at most two highlighted intervals occur at each length.}
\label{fig:dyadic-cover}
\end{figure}

Apply this construction separately to every constant segment of $\mathbf u^{(S)}$ restricted to $\{\tau_j,\ldots,n\}$. There are at most $S+1$ such segments. Consequently, the resulting family is precisely $\calD_j(n)$ and has cardinality at most $(S+1)L$.
\end{proof}

\subsection{Proof of \pref{lem:random-rate-recovery}}
\label{app:random-rate-recovery}

We next prove \pref{lem:random-rate-recovery}, which bounds the sum $\sum_{J\in\calD_j(n)}A_J(n)$ in terms of $W_j(n)$. The key is that conditioning on $\cF_j^{\mathrm{path}}$ fixes the cover $\calD_j(n)$ and these advantages while leaving the interval-rate draws independent. The events $\eta_JA_J(n)\geq4L$ are therefore conditionally independent across intervals.

\begin{proof}[Proof of \pref{lem:random-rate-recovery}]
Fix $n\in\{\tau_j,\ldots,T\}$ and condition on $\cF_j^{\mathrm{path}}$. This fixes $\calD_j(n)$ and $(A_J(n))_{J\in\calD_j(n)}$, while the $U_J$ remain independent uniform variables. For $J\in\calD_j(n)$, define $v_J(n)\triangleq \frac{\alpha A_J(n)}{d_j}$. Since $\eta_J=\alpha\eta_j/U_J$ and $d_j=4L/\eta_j$, we have $\1\{\eta_JA_J(n)\geq4L\}=\1\{U_J\leq v_J(n)\}$. Denote this indicator by $X_J(n)\triangleq\1\{U_J\leq v_J(n)\}$. Hence
\begin{align}
\Pp(X_J(n)=1\mid\cF_j^{\mathrm{path}})
=\min\{1,(v_J(n))_+\}.
\label{eq:useful-probability}
\end{align}

For every fixed $\lambda\in(0,1]$ and $J\in\calD_j(n)$, we claim that
\begin{align}
\E\left[
e^{\lambda v_J(n)-v_J(n)X_J(n)}
\mid\cF_j^{\mathrm{path}}
\right]
\leq e^{\lambda^2/2}.
\label{eq:signed-recovery-mgf}
\end{align}
If $v_J(n)\leq0$ or $v_J(n)\geq1$, then $X_J(n)=0$ or $X_J(n)=1$, respectively, and the exponent is nonpositive. For $0<v_J(n)<1$, \pref{eq:useful-probability} gives
\begin{align*}
\E\left[
e^{\lambda v_J(n)-v_J(n)X_J(n)}
\mid\cF_j^{\mathrm{path}}
\right]
=e^{\lambda v_J(n)}
  \left[1-v_J(n)(1-e^{-v_J(n)})\right]\leq e^{\lambda v_J(n)-v_J(n)^2/2}
\leq e^{\lambda^2/2}.
\end{align*}
The first inequality uses $1-x\leq e^{-x}$ and $1-e^{-v_J(n)}\geq v_J(n)-v_J(n)^2/2\geq v_J(n)/2$. The second follows from the AM--GM inequality $\lambda v_J(n)-\frac{v_J(n)^2}{2}\leq\frac{\lambda^2}{2}$.

We also have $W_j(n)/d_j=\sum_{J\in\calD_j(n)}v_J(n)X_J(n)$. Conditional on $\cF_j^{\mathrm{path}}$, the $v_J(n)$ are fixed and each $X_J(n)$ depends only on $U_J$. Set $r(n)\triangleq\sum_{J\in\calD_j(n)}v_J(n) =(\alpha/d_j)\sum_{J\in\calD_j(n)}A_J(n)$. The conditional independence of the $U_J$ given $\cF_j^{\mathrm{path}}$ therefore gives
\begin{align}
\E\left[
\exp\left\{\lambda r(n)-\frac{W_j(n)}{d_j}\right\}
\mid\cF_j^{\mathrm{path}}
\right]
=\prod_{J\in\calD_j(n)}\E\left[
  e^{\lambda v_J(n)-v_J(n)X_J(n)}
  \mid\cF_j^{\mathrm{path}}
\right]\leq e^{(S+1)L\lambda^2/2}.
\label{eq:recovery-product-mgf}
\end{align}
The inequality uses \pref{eq:signed-recovery-mgf} and $|\calD_j(n)|\leq(S+1)L$ from \pref{lem:dyadic-cover}.

If $r(n)\leq0$, then \pref{eq:advantage-from-selected} follows immediately from $W_j(n)\geq0$. Otherwise, choose $\lambda^\star(n)=\min\{1,r(n)/((S+1)L)\}$. Because $r(n)$ is fixed under the conditioning on $\cF_j^{\mathrm{path}}$, this choice is also fixed. Applying Markov's inequality to \pref{eq:recovery-product-mgf} with $\lambda=\lambda^\star(n)$ therefore gives, with conditional probability at least $1-\delta_0$,
\begin{align}
\lambda^\star(n)r(n)-\frac{(S+1)L(\lambda^\star(n))^2}{2}
\leq\frac{W_j(n)}{d_j}+\log\frac1{\delta_0}
\triangleq h(n).
\label{eq:optimized-recovery-event}
\end{align}
For $0<r(n)\leq(S+1)L$, substituting $\lambda^\star(n)=r(n)/((S+1)L)$ gives $r(n)^2/(2(S+1)L)\leq h(n)$, hence $r(n)\leq\sqrt{2(S+1)Lh(n)}$. For $r(n)>(S+1)L$, substituting $\lambda^\star(n)=1$ gives $r(n)-(S+1)L/2\leq h(n)$. Thus $h(n)>(S+1)L/2$ and $r(n)\leq h(n)+(S+1)L/2\leq h(n)+\sqrt{2(S+1)Lh(n)}$. Consequently, in both cases,
\[
r(n)\leq h(n)+\sqrt{2(S+1)Lh(n)}.
\]
{Substituting the definitions of $r(n)$ and $h(n)$, using} $\log(1/\delta_0)=3\log T\leq L/2$ and multiplying by $d_j/\alpha$, we obtain
\[
\sum_{J\in\calD_j(n)}A_J(n)
\leq\frac1\alpha\left[
W_j(n)+\frac{d_jL}{2}
+\sqrt{2(S+1)Ld_j\left(W_j(n)+\frac{d_jL}{2}\right)}
\right],
\]
which is \pref{eq:advantage-from-selected}. This bound holds with conditional probability at least $1-\delta_0$ given $\cF_j^{\mathrm{path}}$. Since $\cF_j\subseteq\cF_j^{\mathrm{path}}$, the tower property gives the same guarantee conditional on $\cF_j$.
\end{proof}

\subsection{Proof of \pref{lem:selected-credit-recovery}}
\label{app:selected-credit-recovery}

In this section, we prove \pref{lem:selected-credit-recovery}, which shows that continuation credit captures a constant fraction of the scaled advantage $W_j(n)$. We apply the credit bound in \pref{lem:shared-credit} to the intervals $J\in\calD_j(n)$ satisfying $\eta_JA_J(n)\geq4L$ and use this condition to control their initialization penalties and comparator arm estimation errors. The remaining terms are bounded by $KT\eta_j/2+4\Gamma_j$.

\begin{proof}[Proof of \pref{lem:selected-credit-recovery}]
Fix a reached epoch $I_j$ and an endpoint $n\in\{\tau_j,\ldots,T\}$, and condition on $\cF_j^{\mathrm{rate}}$. We use the credit bound in \pref{lem:shared-credit}, which lower-bounds continuation credit by the total scaled advantage of a disjoint interval family, minus the comparison penalties and error terms in \pref{eq:selected-credit-prefix}. Specifically, we apply the lemma with $\delta=\delta_0$ to the family
\[
\mathcal S_j=\mathcal S_j(n)
\triangleq\{J\in\calD_j(n):\eta_JA_J(n)\geq4L\},
\]
 using the comparator arm $a_J$ on each interval $J\in\mathcal S_j(n)$.

We first verify that $\mathcal S_j(n)$ satisfies the hypotheses of \pref{lem:shared-credit}. By construction, $\calD_j(n)\subseteq\mathcal J_j$ is a disjoint family and $u_t^{(S)}=e_{a_J}$ for every $t\in J$, so these properties also hold for $\mathcal S_j(n)$. Moreover, $\cF_j^{\mathrm{path}}$ determines $\calD_j(n)$ and $A_J(n)$, while $\cF_j^{\mathrm{rate}}$ additionally reveals $\eta_J$. Consequently, the family $\mathcal S_j(n)$ defined by the threshold $\eta_JA_J(n)\geq4L$ is $\cF_j^{\mathrm{rate}}$-measurable. By \pref{app:continuation-conditioning}, the challenge-action seeds remain independent and uniform under this conditioning, preserving the conditional sampling law used in \pref{lem:shared-credit}. For this fixed endpoint $n$, we hold $\mathcal S_j(n)$ fixed throughout the application of \pref{lem:shared-credit}. {} {}

Applying \pref{eq:selected-credit-prefix} with the above parameters and evaluating its prefix bound at $n$ therefore gives, with conditional probability at least $1-\delta_0$,
\begin{align}\label{eq:credit}
C_{j,n}^{\mathrm{cont}}
&\geq
\alpha\sum_{J\in\mathcal S_j(n)}A_J(n)
-\sum_{J\in\mathcal S_j(n)}
 \frac{\alpha}{\eta_J}
 \log\frac{6TK}{\delta_0\pi_{a_J}}
-\frac2{\eta_j}-4\Gamma_j
-2\log\frac3{\delta_0}
-2\alpha\sqrt{T\log\frac3{\delta_0}},
\end{align}
where $\pi$ is the initial distribution of each interval subroutine over the reference coordinate $0$ and the $K$ arms: \( \pi_0=\frac{T}{T+1}, \pi_i=\frac1{K(T+1)} \) for $i\in[K]$. These weights follow from the subroutine initialization in \pref{alg:main}.

It remains to bound the arm-comparison penalties and the three scalar error terms. Using the value of $\pi_{a_J}$ and the definitions of $\delta_0$ and $L$, we obtain that
\begin{align*}
\frac{\alpha}{\eta_J}\log\frac{6TK}{\delta_0\pi_{a_J}}
&=\frac{\alpha}{\eta_J}\log\bigl(6K^2T^4(T+1)\bigr)
\leq \frac{\alpha L}{\eta_J}\leq \frac{\alpha A_J(n)}{4}.
\end{align*}
Here we use $K<T$, $L=\lceil8\log T\rceil$, and $\eta_JA_J(n)\geq4L$ for every $J\in\mathcal S_j(n)$.

Summing over these intervals and substituting into \pref{eq:credit} gives
\begin{align}
C_{j,n}^{\mathrm{cont}}
&\geq
\frac{3\alpha}{4}
\sum_{J\in\mathcal S_j(n)}A_J(n)
-4\Gamma_j-\frac2{\eta_j}
-2\log\frac3{\delta_0}
-2\alpha\sqrt{T\log\frac3{\delta_0}}.
\label{eq:selected-credit-before-constants}
\end{align}
The parameter choices give $\eta_j\geq L/\sqrt{KT}$ and $\alpha=1/L^2$. Direct calculation shows that $\log(3/\delta_0)=\log(3T^3)\leq4\log T\leq L/2$. Consequently, we have

\begin{align*}
\frac{\frac2{\eta_j}+2\log\frac3{\delta_0}
+2\alpha\sqrt{T\log\frac3{\delta_0}}}{KT\eta_j}
\leq\frac2{L^2}+\frac1{\sqrt{KT}}
+\frac{\sqrt2}{\sqrt K L^{5/2}}\leq\frac2{81}+\frac1{\sqrt6}+\frac1{243}<\frac12.
\end{align*}
The first inequality uses $KT\eta_j^2\geq L^2$, $KT\eta_j\geq L\sqrt{KT}$, $\log(3/\delta_0)\leq L/2$, and $\alpha=1/L^2$. The second uses $K\geq2$ and $T\geq3$. These three terms therefore sum to at most $KT\eta_j/2$. By the definitions of $\calS_j(n)$ and $W_j(n)$, we have
\[
\alpha\sum_{J\in\mathcal S_j(n)}A_J(n)
=
\alpha\sum_{J\in\calD_j(n)}A_J(n)\1\{\eta_JA_J(n)\geq4L\}
=
W_j(n)\geq0.
\]
Substituting into \pref{eq:selected-credit-before-constants} and rearranging the terms give the following inequality:
\[
C_{j,n}^{\mathrm{cont}}
\geq
\frac34W_j(n)-{\frac{KT\eta_j}{2}}-4\Gamma_j,
\]
which proves \pref{eq:credit-controls-selected} with conditional probability at least $1-\delta_0$ given $\cF_j^{\mathrm{rate}}$. Since $\cF_j\subseteq\cF_j^{\mathrm{rate}}$, the tower property yields the same guarantee for $\cF_j$.
\end{proof}

\subsection{Proof of \pref{lem:continuation-charge}}
\label{app:continuation-charge}

We next prove \pref{lem:continuation-charge}, which bounds the expected value of $\Gamma_j$.

\begin{proof}[Proof of \pref{lem:continuation-charge}]
Condition on $\cF_j$, which fixes $\mathcal J_j$ and $\eta_j$ while leaving the $U_J$ independent and uniform. For $U\sim\operatorname{Unif}(0,1)$ and $c>0$, direct integration gives
\begin{align}
\E\left[\min\left\{\frac cU,1\right\}\right]
&=
\begin{cases}
c+\displaystyle\int_c^1\frac c u\,du
=c(1+\log(1/c)),&0<c<1,\\
1,&c\geq1,
\end{cases}
\notag\\
&\leq c\left(1+\logp\frac1c\right),
\label{eq:uniform-min}
\end{align}
where we recall that $\log^+(z)=\max\{0,\log z\}$. Applying \pref{eq:uniform-min} with $c=K\eta_j$ and using $\eta_J=\alpha\eta_j/U_J$, we obtain
\begin{align*}
\E_j[\min\{K\eta_J,\alpha\}]
=\alpha\E_j\left[
\min\left\{\frac{K\eta_j}{U_J},1\right\}\right]\leq\alpha K\eta_j
\left(1+\logp\frac1{K\eta_j}\right)
\leq{{\frac32\alpha K\eta_j\log T}}.
\end{align*}
For the last inequality, the rate lower bound gives $1/(K\eta_j)\leq\sqrt{T/K}/L\leq\sqrt T$, hence $1+\logp(1/(K\eta_j))\leq1+\frac{1}{2}\log T\leq\frac{3}{2}\log T$. At each dyadic scale, the canonical intervals have total length at most $T$, and the number of scales is at most $1+\log_2T\leq\frac52\log T$. Therefore, we have
\begin{align*}
\E_j[\Gamma_j]
=\sum_{J\in\mathcal J_j}|J|\,
  \E_j[\min\{K\eta_J,\alpha\}]\leq{{\frac32\alpha K\eta_j\log T}}\sum_{J\in\mathcal J_j}|J|\leq{{4\alpha KT\eta_j(\log T)^2}}
\leq{\frac{KT\eta_j}{16}},
\end{align*}
where the last inequality uses $\alpha=1/L^2$ and $L\geq8\log T$. This finishes the proof.
\end{proof}

\subsection{Useful Lemmas}
\label{app:auxiliary-estimates}
The following lemmas control the OMD update and loss-estimation terms in \pref{lem:shared-credit}.
\subsubsection{OMD Analysis}

The first lemma bounds an interval subroutine's regret against any $u\in\cX$ by the standard entropy-OMD argument. The completion from \pref{sec:prelim} expresses the update on $\Delta_{K+1}$, allowing us to telescope the KL divergence to the comparator.

\begin{lemma}
\label{lem:subsimplex-omd}
Fix an interval subroutine $J$ in a fixed-rate continuation of \pref{alg:main} and an endpoint $n\in[T]$. On a challenge round, define $y_{t,J,0}\triangleq\ip{q_t}{\widehat\ell_{t,J}}$ and $y_{t,J,i}\triangleq\widehat\ell_{t,J,i}$ for $i\in[K]$. {} Then, for every $u\in\cX$,
\begin{align}
\sum_{\substack{t\in J,\ t\leq n\\b_t=0}}
\ip{x_{t,J}-u}{z_{t,J}}
\leq
\frac{\KL(\bar u\|\pi)}{\eta_J}
+\frac{\eta_J}{2}
\sum_{\substack{t\in J,\ t\leq n\\b_t=0}}
\ip{\bar x_{t,J}}{y_{t,J}^2},
\label{eq:subsimplex-omd-bound}
\end{align}
where $y_{t,J}^2$ denotes coordinatewise squaring. The completed vectors and initial distribution $\pi$ are defined in \pref{sec:prelim} and \pref{app:continuation-conditioning}, respectively.
\end{lemma}

\begin{proof}
If the prefix contains no challenge round of $J$, both sums vanish and the bound follows from nonnegativity of the KL divergence. By the definitions of the completion and $z_{t,J}$, for every $x\in\cX$,
\[
\ip{\bar x}{y_{t,J}}
=(1-\|x\|_1)y_{t,J,0}
+\sum_{i=1}^Kx_i\widehat\ell_{t,J,i}
=y_{t,J,0}+\ip{x}{z_{t,J}}.
\]
The first term is independent of $x$, and the completed regularizer is negative entropy on $\Delta_{K+1}$. Thus, on each challenge round, the OMD update becomes \( \bar x_{t+1,J,k} = \frac{\bar x_{t,J,k}e^{-\eta_Jy_{t,J,k}}}{\sum_{r=0}^K\bar x_{t,J,r}e^{-\eta_Jy_{t,J,r}}}, \) for $k=0,\ldots,K$. Since $y_{t,J,k}\geq0$, applying $e^{-v}\leq1-v+v^2/2$ for $v\geq0$ and $\log z\leq z-1$ for $z>0$ gives
\[
\log\left(\sum_{r=0}^K\bar x_{t,J,r}e^{-\eta_Jy_{t,J,r}}\right)
\leq
-\eta_J\ip{\bar x_{t,J}}{y_{t,J}}
+\frac{\eta_J^2}{2}
\ip{\bar x_{t,J}}{y_{t,J}^2}.
\]
The exponential-weights update therefore yields the standard KL-potential bound
\begin{align*}
\KL(\bar u\|\bar x_{t+1,J})
-\KL(\bar u\|\bar x_{t,J})
&=
\eta_J\ip{\bar u}{y_{t,J}}
+\log\left(\sum_{r=0}^K\bar x_{t,J,r}e^{-\eta_Jy_{t,J,r}}\right)\\
&\leq
-\eta_J\ip{\bar x_{t,J}-\bar u}{y_{t,J}}
+\frac{\eta_J^2}{2}
\ip{\bar x_{t,J}}{y_{t,J}^2}.
\end{align*}
Rearranging and summing over $t\in J$ with $t\leq n$ and $b_t=0$ telescopes the KL terms. The initial completed distribution of the subroutine is $\pi$, and the terminal KL divergence is nonnegative. Hence
\begin{align*}
\sum_{\substack{t\in J,\ t\leq n\\b_t=0}}
\ip{\bar x_{t,J}-\bar u}{y_{t,J}}
\leq
\frac{\KL(\bar u\|\pi)}{\eta_J}
+\frac{\eta_J}{2}
\sum_{\substack{t\in J,\ t\leq n\\b_t=0}}
\ip{\bar x_{t,J}}{y_{t,J}^2}.
\end{align*}
Finally, using $ \ip{\bar x_{t,J}-\bar u}{y_{t,J}} = \ip{x_{t,J}-u}{z_{t,J}} $ gives \pref{eq:subsimplex-omd-bound} and proves the lemma.
\end{proof}

\subsubsection{\texorpdfstring{Relating Credit to Interval Advantages}{Relating Credit to Interval Advantages}}
\label{app:shared-credit}

The next technical lemma relates continuation credit to the advantages of a disjoint interval family $\mathcal S_j\subseteq\mathcal J_j$. In the proof of \pref{lem:selected-credit-recovery}, for each fixed endpoint $n$, we apply the following lemma with $\mathcal S_j=\{J\in\calD_j(n):\eta_JA_J(n)\geq4L\}$, and the remaining subroutines are indexed by $J\in\mathcal J_j\setminus\mathcal S_j$. To prove the lemma, we apply the OMD bound with the comparator arm on each $J\in\mathcal S_j$ and with $\mathbf0$ on each $J\in\mathcal J_j\setminus\mathcal S_j$. We then control the correction from implicit exploration and the OMD update terms by a martingale bound, the arm estimates by the implicit-exploration inequality of \citet[Corollary~1]{neu2015explore}, and the fluctuations in the main learner's loss estimates by Freedman's inequality.

\begin{lemma}
\label{lem:shared-credit}
Consider the fixed-rate continuation of a reached epoch $I_j$ of \pref{alg:main}, and let $\mathcal S_j\subseteq\mathcal J_j$ be an $\cF_j^{\mathrm{rate}}$-measurable collection of disjoint intervals such that $u_t^{(S)}=e_{a_J}$ on every $J\in\mathcal S_j$. For every $\delta\in(0,1)$, conditional on $\cF_j^{\mathrm{rate}}$, with probability at least $1-\delta$, simultaneously for all $n\in\{\tau_j,\ldots,T\}$,
\begin{align}
C_{j,n}^{\mathrm{cont}}
&\geq
\alpha\sum_{J\in\mathcal S_j}A_J(n)
-\sum_{J\in\mathcal S_j}\frac{\alpha}{\eta_J}
 \log\frac{6TK}{\delta\pi_{a_J}}
-\frac2{\eta_j}-4\Gamma_j-2\log\frac3\delta
-2\alpha\sqrt{T\log\frac3\delta}.
\label{eq:selected-credit-prefix}
\end{align}
Here $\pi$ is the completed initial distribution from \pref{app:continuation-conditioning} and $A_J(n)$ is defined in \pref{eq:local-advantage}.

\end{lemma}

\begin{proof}
Condition on $\cF_j^{\mathrm{rate}}$, which fixes the family $\mathcal S_j$ and the rates. All quantities below refer to the continuation. By \pref{app:continuation-conditioning}, on challenge rounds $p_t$ is $\cV_{j,t}^{\mathrm{chal}}$-measurable and $\Pp(i_t=i\mid\cV_{j,t}^{\mathrm{chal}})=p_{t,i}$.

To apply the OMD bound, we first relate the credit sample to the subroutines' estimated losses. For each active interval, its completed decision induces the arm distribution $w_{t,J}\triangleq(1-\|x_{t,J}\|_1)q_t+x_{t,J}$. For $J\in\mathcal S_j\cap\cA_t$, we have $\ip{q_t-e_{a_J}}{\ell_t}=\ip{q_t-w_{t,J}}{\ell_t}+\ip{w_{t,J}-e_{a_J}}{\ell_t}$, which separates the subroutine's improvement from its regret. For $J\in\cA_t\setminus\mathcal S_j$, comparing with $q_t$ makes the analogous two terms sum to zero. The mixture in \pref{eq:challenge-distribution} gives \( p_t=(1-\alpha|\cA_t|)q_t +\alpha\sum_{J\in\cA_t}w_{t,J}\). Since $1-\alpha|\cA_t|\geq1/2$, this gives $\alpha\sum_{J\in\cA_t}w_{t,J,i}\leq p_{t,i}$ and $p_{t,i}\geq q_{t,i}/2$ for all $i\in[K]$.

On a challenge round $t$, using the definition of $Z_t=\ell_{t,i_t}\big(\frac{q_{t,i_t}}{p_{t,i_t}}-1\big)$ gives
\begin{align}
Z_t
&=
\frac{\ell_{t,i_t}(q_{t,i_t}-p_{t,i_t})}
     {p_{t,i_t}}
\notag\\
&=
\alpha\sum_{J\in\cA_t}
\frac{\ell_{t,i_t}
\bigl(\|x_{t,J}\|_1q_{t,i_t}-x_{t,J,i_t}\bigr)}
{p_{t,i_t}}
\notag\\
&=
\alpha\sum_{J\in\cA_t}
\frac{\ell_{t,i_t}
\bigl(\|x_{t,J}\|_1q_{t,i_t}-x_{t,J,i_t}\bigr)}
{p_{t,i_t}+\eta_J}+
\alpha\sum_{J\in\cA_t}
\frac{\eta_J\ell_{t,i_t}
\bigl(\|x_{t,J}\|_1q_{t,i_t}-x_{t,J,i_t}\bigr)}
{p_{t,i_t}(p_{t,i_t}+\eta_J)}
\notag\\
&\geq
-\alpha\sum_{J\in\cA_t}
\ip{x_{t,J}}{z_{t,J}}
-\alpha\sum_{J\in\cA_t}
\frac{\eta_Jx_{t,J,i_t}\ell_{t,i_t}}
{p_{t,i_t}(p_{t,i_t}+\eta_J)},
\label{eq:credit-exact-x}
\end{align}
where the last inequality drops the nonnegative term and uses the definition of $z_{t,J} =\widehat\ell_{t,J} -\ip{q_t}{\widehat\ell_{t,J}}\mathbf 1$. Moreover, for each active interval $J\in\cA_t$, recall the notation $y_{t,J,0}\triangleq\ip{q_t}{\widehat\ell_{t,J}}$ and $y_{t,J,i}\triangleq\widehat\ell_{t,J,i}$ for $i\in[K]$ defined in~\pref{lem:subsimplex-omd}. Then, by the definition of the completed vector $\bar x_{t,J}$ and Jensen's inequality, we have
\begin{align} \frac\alpha2\sum_{J\in\cA_t}\eta_J \ip{\bar x_{t,J}}{y_{t,J}^2} &= \frac\alpha2\sum_{J\in\cA_t}\eta_J \left[ (1-\|x_{t,J}\|_1) \ip{q_t}{\widehat\ell_{t,J}}^2 + \ip{x_{t,J}}{\widehat\ell_{t,J}^2} \right] \notag\\ &\leq \frac\alpha2\sum_{J\in\cA_t}\eta_J \left[ (1-\|x_{t,J}\|_1) \ip{q_t}{\widehat\ell_{t,J}^2} + \ip{x_{t,J}}{\widehat\ell_{t,J}^2} \right] \notag\\ &= \frac\alpha2\sum_{J\in\cA_t}\eta_J \ip{w_{t,J}}{\widehat\ell_{t,J}^2} \notag\\ &= \frac\alpha2\sum_{J\in\cA_t} \frac{\eta_Jw_{t,J,i_t}\ell_{t,i_t}^2} {(p_{t,i_t}+\eta_J)^2}, \label{eq:credit-curvature} \end{align} where the last equality uses the fact that $\widehat\ell_{t,J}$ has only one nonzero coordinate  $\widehat\ell_{t,J,i_t} = \ell_{t,i_t}/(p_{t,i_t}+\eta_J)$.

Applying \pref{lem:subsimplex-omd} on an interval $J$ gives, for every $u\in\cX$,
\begin{align*}
-\alpha
\sum_{\substack{t\in J,\ t\leq n\\b_t=0}}
\ip{x_{t,J}}{z_{t,J}}
\geq
-\alpha
\sum_{\substack{t\in J,\ t\leq n\\b_t=0}}
\ip{u}{z_{t,J}}
-\frac{\alpha}{\eta_J}\KL(\bar u\|\pi)
-\frac{\alpha\eta_J}{2}
\sum_{\substack{t\in J,\ t\leq n\\b_t=0}}
\ip{\bar x_{t,J}}{y_{t,J}^2}.
\end{align*}
For an interval $J\in\mathcal S_j$, taking $u=e_{a_J}$ gives $\KL(\overline{e_{a_J}}\|\pi)=\log(1/\pi_{a_J})$ and $ -\ip{e_{a_J}}{z_{t,J}}= \ip{q_t}{\widehat\ell_{t,J}} -\widehat\ell_{t,J,a_J}.$ Therefore, we have
\begin{align*}
-\alpha
\sum_{\substack{t\in J,\ t\leq n\\b_t=0}}
\ip{x_{t,J}}{z_{t,J}}
&\geq
\alpha
\sum_{\substack{t\in J,\ t\leq n\\b_t=0}}
\left(
\ip{q_t}{\widehat\ell_{t,J}}
-\widehat\ell_{t,J,a_J}
\right)
-\frac{\alpha}{\eta_J}\log\frac1{\pi_{a_J}}
-\frac{\alpha\eta_J}{2}
\sum_{\substack{t\in J,\ t\leq n\\b_t=0}}
\ip{\bar x_{t,J}}{y_{t,J}^2}.
\end{align*}
For $J\in\mathcal J_j\setminus\mathcal S_j$, taking $u=\mathbf0$ instead gives $\KL(\overline{\mathbf0}\|\pi)=\log(1/\pi_0)$ and $\ip{\mathbf0}{z_{t,J}}=0$, and hence
\begin{align*}
-\alpha
\sum_{\substack{t\in J,\ t\leq n\\b_t=0}}
\ip{x_{t,J}}{z_{t,J}}
&\geq
-\frac{\alpha}{\eta_J}\log\frac1{\pi_0}
-\frac{\alpha\eta_J}{2}
\sum_{\substack{t\in J,\ t\leq n\\b_t=0}}
\ip{\bar x_{t,J}}{y_{t,J}^2}.
\end{align*}
Summing these bounds over the interval subroutines gives
\begin{align}
&-\alpha
\sum_{\substack{t=\tau_j\\b_t=0}}^n
\sum_{J\in\cA_t}
\ip{x_{t,J}}{z_{t,J}}\nonumber\\
&\geq
\alpha\sum_{J\in\mathcal S_j}
\sum_{\substack{t\in J,\ t\leq n\\b_t=0}}
\left(
\ip{q_t}{\widehat\ell_{t,J}}
-\widehat\ell_{t,J,a_J}
\right)
-\sum_{J\in\mathcal S_j}
\frac{\alpha}{\eta_J}\log\frac1{\pi_{a_J}}
 \nonumber\\
&\qquad
-\sum_{J\in\mathcal J_j\setminus\mathcal S_j}
\frac{\alpha}{\eta_J}\log\frac1{\pi_0}-\frac{\alpha}{2}
\sum_{\substack{t=\tau_j\\b_t=0}}^n
\sum_{J\in\cA_t}
\eta_J\ip{\bar x_{t,J}}{y_{t,J}^2}\nonumber\\
&\geq \alpha\sum_{J\in\mathcal S_j}
\sum_{\substack{t\in J,\ t\leq n\\b_t=0}}
\left(
\ip{q_t}{\widehat\ell_{t,J}}
-\widehat\ell_{t,J,a_J}
\right)
-\sum_{J\in\mathcal S_j}
\frac{\alpha}{\eta_J}\log\frac1{\pi_{a_J}}
-\frac{2}{\eta_j}-\frac{\alpha}{2}
\sum_{\substack{t=\tau_j\\b_t=0}}^n
\sum_{J\in\cA_t}
\eta_J\ip{\bar x_{t,J}}{y_{t,J}^2},\label{eq:adv}
\end{align}
where the last inequality uses $|\mathcal J_j|\leq 2T$, $\log(1/\pi_0)=\log(1+1/T)\leq1/T$, $\eta_J=\frac{\alpha\eta_j}{U_J}\geq\alpha\eta_j$, and
\begin{align*}
\sum_{J\in\mathcal J_j\setminus\mathcal S_j}
\frac{\alpha}{\eta_J}\log\frac1{\pi_0}
&\leq
\sum_{J\in\mathcal J_j}
\frac{\alpha}{\eta_J}\log\frac1{\pi_0}\leq
\frac{|\mathcal J_j|}{\eta_j}\log\frac1{\pi_0}
\leq 
\frac2{\eta_j}.
\end{align*}

We now combine this OMD comparison with the preceding credit and OMD update bounds. Since $p_t=q_t$ and hence $Z_t=0$ on main rounds, summing \pref{eq:credit-exact-x} over the challenge rounds up to $n$ gives
\begin{align}\label{eq:bonus-cont}
C_{j,n}^{\mathrm{cont}}
=\sum_{t=\tau_j}^nZ_t
=\sum_{\substack{t=\tau_j\\b_t=0}}^nZ_t
\geq
-\alpha
\sum_{\substack{t=\tau_j\\b_t=0}}^n
\sum_{J\in\cA_t}
\ip{x_{t,J}}{z_{t,J}}
-\alpha
\sum_{\substack{t=\tau_j\\b_t=0}}^n
\sum_{J\in\cA_t}
\frac{\eta_Jx_{t,J,i_t}\ell_{t,i_t}}
{p_{t,i_t}(p_{t,i_t}+\eta_J)}.
\end{align}
By \pref{eq:credit-curvature}, the OMD update term from \pref{lem:subsimplex-omd} satisfies
\begin{align}\label{eq:curve-chal}
\frac{\alpha}{2}
\sum_{\substack{t=\tau_j\\b_t=0}}^n
\sum_{J\in\cA_t}
\eta_J\ip{\bar x_{t,J}}{y_{t,J}^2}
&\leq
\frac{\alpha}{2}
\sum_{\substack{t=\tau_j\\b_t=0}}^n
\sum_{J\in\cA_t}
\frac{\eta_Jw_{t,J,i_t}\ell_{t,i_t}^2}
{(p_{t,i_t}+\eta_J)^2}.
\end{align}
Plugging \pref{eq:adv} and \pref{eq:curve-chal} into \pref{eq:bonus-cont} and using $0\leq\ell_{t,i_t}\leq1$ gives
\begin{align}
C_{j,n}^{\mathrm{cont}}
&\geq
\alpha\sum_{J\in\mathcal S_j}
\sum_{\substack{t\in J,\ t\leq n\\b_t=0}}
\left(
\ip{q_t}{\widehat\ell_{t,J}}
-\widehat\ell_{t,J,a_J}
\right)
-\sum_{J\in\mathcal S_j}
\frac{\alpha}{\eta_J}\log\frac1{\pi_{a_J}}
-\frac2{\eta_j}
-\sum_{t=\tau_j}^nD_t,
\label{eq:credit-after-omd}
\end{align}
where $D_t$ bounds the denominator correction and OMD update terms:
\begin{align}
D_t\triangleq\alpha\sum_{J\in\cA_t}\eta_J
\left(
\frac{x_{t,J,i_t}}
 {p_{t,i_t}(p_{t,i_t}+\eta_J)}
+\frac{w_{t,J,i_t}}
 {2(p_{t,i_t}+\eta_J)^2}
\right)
\label{eq:combined-credit-cost}
\end{align}
for each challenge round $t$. Set $D_t=0$ on main rounds.

By \pref{lem:combined-cost-control} with $\varepsilon=\delta/3$, with conditional probability at least $1-\delta/3$, simultaneously for all $n$,
\begin{align}
\sum_{t=\tau_j}^nD_t
\leq
3\Gamma_j+\frac32\log\frac3\delta.
\label{eq:combined-cost-prefix}
\end{align}
Moreover, by \pref{lem:selected-estimate-control} with $\varepsilon=\delta/3$, with conditional probability at least $1-2\delta/3$, simultaneously for all $n$,
\begin{align}
&\alpha\sum_{J\in\mathcal S_j}
\sum_{\substack{t\in J,\ t\leq n\\b_t=0}}
\left(
\ip{q_t}{\widehat\ell_{t,J}}
-\widehat\ell_{t,J,a_J}
\right)\nonumber\\
&\geq
\alpha\sum_{J\in\mathcal S_j}A_J(n)
-\sum_{J\in\mathcal S_j}
\frac{\alpha}{\eta_J}
\log\frac{6TK}{\delta}
-\Gamma_j
-2\alpha\sqrt{T\log\frac3\delta}
-\alpha\log\frac3\delta.
\label{eq:selected-estimate-control-use}
\end{align}
Intersecting these two events and substituting them into \pref{eq:credit-after-omd} gives, simultaneously for all $n$,
\begin{align*}
C_{j,n}^{\mathrm{cont}}
&\geq
\alpha\sum_{J\in\mathcal S_j}A_J(n)
-\sum_{J\in\mathcal S_j}
\frac{\alpha}{\eta_J}
\left(
\log\frac1{\pi_{a_J}}
+\log\frac{6TK}{\delta}
\right)
-\frac2{\eta_j}
-4\Gamma_j
-\left(\frac32+\alpha\right)\log\frac3\delta
-2\alpha\sqrt{T\log\frac3\delta}\\
&\geq
\alpha\sum_{J\in\mathcal S_j}A_J(n)
-\sum_{J\in\mathcal S_j}
\frac{\alpha}{\eta_J}
\log\frac{6TK}{\delta\pi_{a_J}}
-\frac2{\eta_j}
-4\Gamma_j
-2\log\frac3\delta
-2\alpha\sqrt{T\log\frac3\delta},
\end{align*}
where the last inequality uses \mbox{$\alpha\leq\frac{1}{2}$}. This finishes the proof.
\end{proof}

The first auxiliary lemma controls the denominator correction and the OMD update terms bounded by $D_t$. We bound the sum of their conditional means using $\Gamma_j$ and apply Freedman's inequality to control their fluctuations uniformly over all continuation prefixes.

\begin{lemma}
\label{lem:combined-cost-control}
Consider the fixed-rate continuation of a reached epoch $I_j$ of \pref{alg:main} and condition on $\cF_j^{\mathrm{rate}}$. For every $\varepsilon\in(0,1)$, with conditional probability at least $1-\varepsilon$, simultaneously for all $n\in\{\tau_j,\ldots,T\}$,
\begin{align}
\sum_{t=\tau_j}^nD_t
\leq
3\Gamma_j+\frac32\log\frac1\varepsilon,
\label{eq:combined-cost-control}
\end{align}
where $D_t$ is defined in \pref{eq:combined-credit-cost}.
\end{lemma}

\begin{proof}
On a challenge round, the first term in \pref{eq:combined-credit-cost} satisfies
\begin{align*}
\alpha\sum_{J\in\cA_t}\eta_J
\frac{x_{t,J,i_t}}
{p_{t,i_t}(p_{t,i_t}+\eta_J)}
&\leq
\frac{\alpha\sum_{J\in\cA_t}x_{t,J,i_t}}
{p_{t,i_t}}
\leq1,
\end{align*}
whereas $\eta/(p+\eta)^2\leq1/(4p)$ gives
\begin{align*}
\frac{\alpha}{2}\sum_{J\in\cA_t}\eta_J
\frac{w_{t,J,i_t}}
{(p_{t,i_t}+\eta_J)^2}
&\leq
\frac{\alpha\sum_{J\in\cA_t}w_{t,J,i_t}}
{8p_{t,i_t}}
\leq\frac18.
\end{align*}
Thus $0\leq D_t\leq9/8$ on challenge rounds, while $D_t=0$ on main rounds. Conditional on $\cV_{j,t}^{\mathrm{chal}}$, averaging over $i_t\sim p_t$ gives
\begin{align*}
\E[D_t\mid\cV_{j,t}^{\mathrm{chal}}]
&=
\alpha\sum_{J\in\cA_t}\eta_J\sum_{i=1}^K
\left(
\frac{x_{t,J,i}}
{p_{t,i}+\eta_J}
+
\frac{p_{t,i}w_{t,J,i}}
{2(p_{t,i}+\eta_J)^2}
\right).
\end{align*}
For every $J\in\cA_t$, the first term is at most both $K\eta_J$ and $\alpha$, using respectively $\alpha x_{t,J,i}\leq p_{t,i}$ and $\eta_J/(p_{t,i}+\eta_J)\leq1$. The second term is at most both $K\eta_J/2$ and $\alpha/8$, using respectively $\alpha w_{t,J,i}\leq p_{t,i}$ and $\eta_Jp_{t,i}/(p_{t,i}+\eta_J)^2\leq1/4$. Consequently,
\begin{align}
\E[D_t\mid\cV_{j,t}^{\mathrm{chal}}]
\leq
\frac32\sum_{J\in\cA_t}\min\{K\eta_J,\alpha\}.
\label{eq:Dt-drift}
\end{align}
For every $t\in\{\tau_j,\ldots,T\}$, let $\mu_t\triangleq\E[D_t\mid\cV_{j,t}^{\mathrm{chal}}]$ on challenge rounds and $\mu_t=0$ on main rounds. Then, for every $n$, we have
\begin{align}
\sum_{t=\tau_j}^n\mu_t
&\leq
\frac32
\sum_{J\in\mathcal J_j}
\left|
\{t\in J:t\leq n,\ b_t=0\}
\right|
\min\{K\eta_J,\alpha\}
\leq
\frac32\Gamma_j.
\label{eq:Dt-drift-prefix}
\end{align}

Denote by $X_t\triangleq D_t-\mu_t$ the centered martingale differences. Since $0\leq D_t\leq9/8$, we have $X_t\leq9/8$ and
\begin{align*}
\E[X_t^2\mid\cV_{j,t}^{\mathrm{chal}}]
=
\Var(D_t\mid\cV_{j,t}^{\mathrm{chal}})\leq
\E[D_t^2\mid\cV_{j,t}^{\mathrm{chal}}]
\leq
\frac98\mu_t.
\end{align*}
Therefore, by \pref{eq:Dt-drift-prefix}, for every $n$,
\begin{align}
\sum_{t=\tau_j}^n
\E[X_t^2\mid\cV_{j,t}^{\mathrm{chal}}]
\leq
\frac98\sum_{t=\tau_j}^n\mu_t
\leq
\frac{27}{16}\Gamma_j.
\label{eq:Dt-variance-prefix}
\end{align}

For any fixed endpoint $n$, applying Freedman's inequality (e.g., \citet[Theorem~1]{beygelzimer2011contextual}) with $R=9/8$ and using \pref{eq:Dt-variance-prefix} gives, with conditional probability at least $1-\varepsilon$,
\begin{align*}
\sum_{t=\tau_j}^nX_t
\leq
\frac98\log\frac1\varepsilon
+\frac32(e-2)\Gamma_j.
\end{align*}

This right-hand side does not depend on $n$, so all prefixes use the same threshold. To obtain a bound for all prefixes simultaneously, denote by $\sigma$ the first $n$ for which $\sum_{t=\tau_j}^nX_t$ exceeds this threshold, with $\sigma=T+1$ if no crossing occurs, and define $\widetilde X_t\triangleq X_t\1\{\sigma\geq t\}$. These increments include the crossing increment and are zero afterwards. Since the inclusion of each increment depends only on earlier rounds, the stopped increments remain martingale differences and satisfy the same variance bound in \pref{eq:Dt-variance-prefix}. Applying the same inequality to $\{\widetilde X_t\}_{t=\tau_j}^T$ at the fixed endpoint $T$ bounds the probability of any crossing by $\varepsilon$, because a crossing makes its terminal sum exceed the same boundary. Thus, with conditional probability at least $1-\varepsilon$, the same bound holds simultaneously for all $n$.

Combining this inequality with \pref{eq:Dt-drift-prefix} yields
\begin{align*}
\sum_{t=\tau_j}^nD_t\leq
\frac32\Gamma_j
+\frac32(e-2)\Gamma_j
+\frac98\log\frac1\varepsilon\leq
3\Gamma_j+\frac32\log\frac1\varepsilon,
\end{align*}
which proves \pref{eq:combined-cost-control}.
\end{proof}

The second auxiliary lemma converts the estimated advantage over a disjoint interval family $\mathcal S_j\subseteq\mathcal J_j$ into its true advantage.

\begin{lemma}
\label{lem:selected-estimate-control}
Consider the fixed-rate continuation of a reached epoch $I_j$ of \pref{alg:main} and condition on $\cF_j^{\mathrm{rate}}$. Let $\mathcal S_j\subseteq\mathcal J_j$ be an $\cF_j^{\mathrm{rate}}$-measurable collection of disjoint intervals such that $u_t^{(S)}=e_{a_J}$ on every $J\in\mathcal S_j$. For every $\varepsilon\in(0,1/2)$, with conditional probability at least $1-2\varepsilon$, simultaneously for all $n\in\{\tau_j,\ldots,T\}$,
\begin{align}
&\alpha\sum_{J\in\mathcal S_j}
\sum_{\substack{t\in J,\ t\leq n\\b_t=0}}
\left(
\ip{q_t}{\widehat\ell_{t,J}}
-\widehat\ell_{t,J,a_J}
\right)\nonumber\\
&\geq
\alpha\sum_{J\in\mathcal S_j}A_J(n)
-\sum_{J\in\mathcal S_j}
\frac{\alpha}{\eta_J}
\log\frac{2TK}{\varepsilon}
-\Gamma_j
-2\alpha\sqrt{T\log\frac1\varepsilon}
-\alpha\log\frac1\varepsilon.
\label{eq:selected-estimate-control}
\end{align}
\end{lemma}

\begin{proof}
We first control the arm-coordinate estimates $\widehat\ell_{t,J}$. Fix $J$ and $n$, and condition on $\cF_j^{\mathrm{rate}}$, which fixes $\eta_J$ and the rounds included in the sum. On challenge rounds, $p_t$ and $\ell_t$ are $\cV_{j,t}^{\mathrm{chal}}$-measurable, and $\Pp(i_t=i\mid\cV_{j,t}^{\mathrm{chal}})=p_{t,i}$ by \pref{app:continuation-conditioning}. Applying \citet[Corollary~1]{neu2015explore} with $\gamma=\eta_J$ and failure probability $\varepsilon/(2T^2)$ gives, simultaneously for all $a\in[K]$,
\begin{align*}
\sum_{\substack{t\in J,\ t\leq n\\b_t=0}}
(\widehat\ell_{t,J,a}-\ell_{t,a})
&\leq
\frac1{2\eta_J}
\log\frac{2KT^2}{\varepsilon}.
\end{align*}
There are fewer than $2T$ canonical intervals and at most $T$ endpoints. A union bound therefore gives, with conditional probability at least $1-\varepsilon$, simultaneously for all $J$, $a$, and $n$,
\begin{align}
\sum_{\substack{t\in J,\ t\leq n\\b_t=0}}
(\widehat\ell_{t,J,a}-\ell_{t,a})
\leq
\frac1{\eta_J}\log\frac{2TK}{\varepsilon}.
\label{eq:arm-ix-control}
\end{align}

We next control the main learner's estimated losses on these intervals. For every prefix $n$, define
\begin{align*}
\widehat Q_n
&\triangleq
\alpha\sum_{J\in\mathcal S_j}
\sum_{\substack{t\in J,\ t\leq n\\b_t=0}}
\ip{q_t}{\widehat\ell_{t,J}},\\
Q_n
&\triangleq
\alpha\sum_{J\in\mathcal S_j}
\sum_{\substack{t\in J,\ t\leq n\\b_t=0}}
\ip{q_t}{\ell_t}.
\end{align*}
For every $J\in\mathcal S_j$ and challenge round $t\in J$, we have
\begin{align*}
\alpha\ip{q_t}{\ell_t}
-\alpha\E\left[
\ip{q_t}{\widehat\ell_{t,J}}
\mid\cV_{j,t}^{\mathrm{chal}}
\right]
&=
\alpha\eta_J
\sum_{i=1}^K
\frac{q_{t,i}\ell_{t,i}}
{p_{t,i}+\eta_J}
\leq
\min\{2\alpha K\eta_J,\alpha\}\leq\min\{K\eta_J,\alpha\}.
\end{align*}
Here the two bounds use $q_t\leq2p_t$, $\alpha\leq1/2$, and $\eta_J/(p_{t,i}+\eta_J)\leq1$. Summing over $J$ and $t$ gives
\begin{align}
Q_n
-\alpha\sum_{J\in\mathcal S_j}
\sum_{\substack{t\in J,\ t\leq n\\b_t=0}}
\E\left[
\ip{q_t}{\widehat\ell_{t,J}}
\mid\cV_{j,t}^{\mathrm{chal}}
\right]
&\leq
\sum_{J\in\mathcal S_j}
|J|\min\{K\eta_J,\alpha\}
\leq
\Gamma_j,
\label{eq:reference-bias-control}
\end{align}
where the last inequality uses $\mathcal S_j\subseteq\mathcal J_j$. We next control the fluctuation of the estimated loss. By the assumption, the intervals in $\mathcal S_j$ are disjoint. Hence, for every $t$, there is at most one $J\in\mathcal S_j$ such that $t\in J$. If such an interval exists, denote it by $J_t$ and define
\begin{align*}
\widehat r_t
&\triangleq
\alpha\1\{b_t=0\}
\ip{q_t}{\widehat\ell_{t,J_t}}.
\end{align*}
If no such interval exists, set $\widehat r_t=0$. Then we have $\widehat Q_n = \sum_{t=\tau_j}^n\widehat r_t.$ Denote by $Y_t\triangleq\E[\widehat r_t\mid\cV_{j,t}^{\mathrm{chal}}]-\widehat r_t$ the centered martingale differences. On a challenge round for which $J_t$ exists,
\begin{align*}
0\leq\widehat r_t
&=
\alpha
\frac{q_{t,i_t}\ell_{t,i_t}}
{p_{t,i_t}+\eta_{J_t}}
\leq
\alpha\frac{q_{t,i_t}}{p_{t,i_t}}
\leq
2\alpha,
\end{align*}
where we used $p_t\geq q_t/2$. Moreover,
\begin{align*}
\E[\widehat r_t\mid\cV_{j,t}^{\mathrm{chal}}]
&=
\alpha\sum_{i=1}^K
\frac{p_{t,i}q_{t,i}\ell_{t,i}}
{p_{t,i}+\eta_{J_t}}
\leq
\alpha\sum_{i=1}^Kq_{t,i}\ell_{t,i}
\leq
\alpha.
\end{align*}
Thus $Y_t\leq\alpha$. Since $\widehat r_t^2\leq2\alpha\widehat r_t$, we also have
\begin{align*}
\E[Y_t^2\mid\cV_{j,t}^{\mathrm{chal}}]
=
\Var(\widehat r_t\mid\cV_{j,t}^{\mathrm{chal}})\leq
2\alpha
\E[\widehat r_t\mid\cV_{j,t}^{\mathrm{chal}}]
-
\E[\widehat r_t\mid\cV_{j,t}^{\mathrm{chal}}]^2
\leq
\alpha^2.
\end{align*}
The same bounds hold trivially when $\widehat r_t=0$. Hence the cumulative conditional variance over every prefix is at most $\alpha^2T$. We apply Freedman's inequality (e.g., \citet[Theorem~1]{beygelzimer2011contextual}) to the stopped sequence at the fixed endpoint $T$, as in \pref{lem:combined-cost-control}, with $R=\alpha$, variance upper bound $\alpha^2T$, and parameter \( \lambda\triangleq \min\left\{\frac1\alpha, \sqrt{\frac{\log(1/\varepsilon)}{(e-2)\alpha^2T}}\right\}. \) This gives, with conditional probability at least $1-\varepsilon$, simultaneously for all $n$,
\begin{align}
\sum_{t=\tau_j}^nY_t
&\leq
(e-2)\lambda\alpha^2T+\frac{\log(1/\varepsilon)}{\lambda}
\leq
2\alpha\sqrt{T\log\frac1\varepsilon}
+\alpha\log\frac1\varepsilon.
\label{eq:reference-fluctuation}
\end{align}
The last inequality follows by substituting either the square-root choice of $\lambda$ or $\lambda=1/\alpha$. Since $ \widehat Q_n = \sum_{t=\tau_j}^n \E[\widehat r_t\mid\cV_{j,t}^{\mathrm{chal}}] - \sum_{t=\tau_j}^nY_t$, combining \pref{eq:reference-bias-control} and \pref{eq:reference-fluctuation} gives
\begin{align}
\widehat Q_n
&\geq
Q_n
-\Gamma_j
-2\alpha\sqrt{T\log\frac1\varepsilon}
-\alpha\log\frac1\varepsilon
\label{eq:selected-reference-lower}
\end{align}
simultaneously for all $n$. Finally, intersecting the events in \pref{eq:selected-reference-lower} and \pref{eq:arm-ix-control}, and applying \pref{eq:arm-ix-control} with $a=a_J$, gives
\begin{align*}
&\alpha\sum_{J\in\mathcal S_j}
\sum_{\substack{t\in J,\ t\leq n\\b_t=0}}
\left(
\ip{q_t}{\widehat\ell_{t,J}}
-\widehat\ell_{t,J,a_J}
\right)\\
&=
\widehat Q_n
-\alpha\sum_{J\in\mathcal S_j}
\sum_{\substack{t\in J,\ t\leq n\\b_t=0}}
\widehat\ell_{t,J,a_J}\\
&\geq
Q_n
-\alpha\sum_{J\in\mathcal S_j}
\sum_{\substack{t\in J,\ t\leq n\\b_t=0}}
\ell_{t,a_J}
-\sum_{J\in\mathcal S_j}
\frac{\alpha}{\eta_J}
\log\frac{2TK}{\varepsilon}
-\Gamma_j
-2\alpha\sqrt{T\log\frac1\varepsilon}
-\alpha\log\frac1\varepsilon\\
&=
\alpha\sum_{J\in\mathcal S_j}
\sum_{\substack{t\in J,\ t\leq n\\b_t=0}}
\ip{q_t-e_{a_J}}{\ell_t}
-\sum_{J\in\mathcal S_j}
\frac{\alpha}{\eta_J}
\log\frac{2TK}{\varepsilon}
-\Gamma_j
-2\alpha\sqrt{T\log\frac1\varepsilon}
-\alpha\log\frac1\varepsilon\\
&=
\alpha\sum_{J\in\mathcal S_j}A_J(n)
-\sum_{J\in\mathcal S_j}
\frac{\alpha}{\eta_J}
\log\frac{2TK}{\varepsilon}
-\Gamma_j
-2\alpha\sqrt{T\log\frac1\varepsilon}
-\alpha\log\frac1\varepsilon.
\end{align*}
The two events each fail with conditional probability at most $\varepsilon$, so their intersection has conditional probability at least $1-2\varepsilon$. This proves \pref{eq:selected-estimate-control}.
\end{proof}
\endgroup

%% file: references.bib
@inproceedings{beygelzimer2011contextual,
  title={Contextual bandit algorithms with supervised learning guarantees},
  author={Beygelzimer, Alina and Langford, John and Li, Lihong and Reyzin, Lev and Schapire, Robert},
  booktitle={International Conference on Artificial Intelligence and Statistics},
  year={2011},
}

@article{abbasi2023newlook,
  author = {Abbasi-Yadkori, Yasin and Gy{\"o}rgy, Andr{\'a}s and Lazi{\'c}, Nevena},
  title = {A New Look at Dynamic Regret for Non-Stationary Stochastic Bandits},
  journal = {Journal of Machine Learning Research},
  volume = {24},
  number = {288},
  year = {2023},
}

@inproceedings{audibert2009minimax,
  author = {Audibert, Jean-Yves and Bubeck, S{\'e}bastien},
  title = {Minimax Policies for Adversarial and Stochastic Bandits},
  booktitle = {Conference on Learning Theory},
  year = {2009},
}

@article{auer2002nonstochastic,
  author = {Auer, Peter and Cesa-Bianchi, Nicol{\`o} and Freund, Yoav and Schapire, Robert E.},
  title = {The Nonstochastic Multiarmed Bandit Problem},
  journal = {SIAM Journal on Computing},
  volume = {32},
  number = {1},
  year = {2002},
}

@inproceedings{auer2019dynamic,
  author = {Auer, Peter and Chen, Yifang and Gajane, Pratik and Lee, Chung-Wei and Luo, Haipeng and Ortner, Ronald and Wei, Chen-Yu},
  title = {Achieving Optimal Dynamic Regret for Non-stationary Bandits without Prior Information},
  booktitle = {Conference on Learning Theory},
  year = {2019},
}

@inproceedings{auer2019tracking,
  author = {Auer, Peter and Gajane, Pratik and Ortner, Ronald},
  title = {Adaptively Tracking the Best Bandit Arm with an Unknown Number of Distribution Changes},
  booktitle = {Conference on Learning Theory},
  year = {2019},
}

@inproceedings{besbes2014stochastic,
  author = {Besbes, Omar and Gur, Yonatan and Zeevi, Assaf},
  title = {Stochastic Multi-Armed-Bandit Problem with Non-stationary Rewards},
  booktitle = {Advances in Neural Information Processing Systems},
  year = {2014},
}

@article{besbes2015nonstationary,
  author = {Besbes, Omar and Gur, Yonatan and Zeevi, Assaf},
  title = {Non-Stationary Stochastic Optimization},
  journal = {Operations Research},
  volume = {63},
  number = {5},
  year = {2015},
}

@inproceedings{chen2019nonstationary,
  author = {Chen, Yifang and Lee, Chung-Wei and Luo, Haipeng and Wei, Chen-Yu},
  title = {A New Algorithm for Non-stationary Contextual Bandits: Efficient, Optimal and Parameter-free},
  booktitle = {Conference on Learning Theory},
  year = {2019},
}

@inproceedings{chen2021combinatorial,
  author = {Chen, Wei and Wang, Liwei and Zhao, Haoyu and Zheng, Kai},
  title = {Combinatorial Semi-Bandit in the Non-Stationary Environment},
  booktitle = {Conference on Uncertainty in Artificial Intelligence},
  year = {2021},
}

@inproceedings{cheung2019learning,
  author = {Cheung, Wang Chi and Simchi-Levi, David and Zhu, Ruihao},
  title = {Learning to Optimize under Non-Stationarity},
  booktitle = {International Conference on Artificial Intelligence and Statistics},
  year = {2019},
}

@inproceedings{cheung2020nonstationarymdp,
  author = {Cheung, Wang Chi and Simchi-Levi, David and Zhu, Ruihao},
  title = {Reinforcement Learning for Non-Stationary {M}arkov Decision Processes: The Blessing of ({M}ore) Optimism},
  booktitle = {International Conference on Machine Learning},
  year = {2020},
}

@article{cheung2021hedging,
  author = {Cheung, Wang Chi and Simchi-Levi, David and Zhu, Ruihao},
  title = {Hedging the Drift: Learning to Optimize Under Nonstationarity},
  journal = {Management Science},
  volume = {68},
  number = {3},
  year = {2022},
}

@inproceedings{faury2021technical,
  author = {Faury, Louis and Russac, Yoan and Abeille, Marc and Calauz{\`e}nes, Cl{\'e}ment},
  title = {A Technical Note on Non-Stationary Parametric Bandits: Existing Mistakes and Preliminary Solutions},
  booktitle = {Algorithmic Learning Theory},
  year = {2021},
}

@inproceedings{luo2018efficient,
  author = {Luo, Haipeng and Wei, Chen-Yu and Agarwal, Alekh and Langford, John},
  title = {Efficient Contextual Bandits in Non-stationary Worlds},
  booktitle = {Conference on Learning Theory},
  year = {2018},
}

@inproceedings{luo2022corralling,
  author = {Luo, Haipeng and Zhang, Mengxiao and Zhao, Peng and Zhou, Zhi-Hua},
  title = {Corralling a Larger Band of Bandits: A Case Study on Switching Regret for Linear Bandits},
  booktitle = {Conference on Learning Theory},
  year = {2022},
}

@inproceedings{mao2021nearoptimal,
  author = {Mao, Weichao and Zhang, Kaiqing and Zhu, Ruihao and Simchi-Levi, David and Basar, Tamer},
  title = {Near-Optimal Model-Free Reinforcement Learning in Non-Stationary Episodic {MDP}s},
  booktitle = {International Conference on Machine Learning},
  year = {2021},
}

@inproceedings{marinov2021pareto,
  author = {Marinov, Teodor Vanislavov and Zimmert, Julian},
  title = {The {Pareto} Frontier of Model Selection for General Contextual Bandits},
  booktitle = {Advances in Neural Information Processing Systems},
  year = {2021},
}

@misc{qian2026simultaneous,
  author = {Qian, Jian and Wei, Chen-Yu},
  title = {Achieving Optimal Static and Dynamic Regret Simultaneously in Bandits with Deterministic Losses},
  year = {2026},
  eprint = {2602.07418},
  howpublished = {arXiv:2602.07418},
}

@inproceedings{rumi2026parameterfree,
  author = {Rumi, Alberto and Jacobsen, Andrew and Cesa-Bianchi, Nicol{\`o} and Vitale, Fabio},
  title = {Parameter-Free Dynamic Regret for Unconstrained Linear Bandits},
  booktitle = {International Conference on Artificial Intelligence and Statistics},
  year = {2026},
}

@inproceedings{russac2019weighted,
  author = {Russac, Yoan and Vernade, Claire and Capp{\'e}, Olivier},
  title = {Weighted Linear Bandits for Non-Stationary Environments},
  booktitle = {Advances in Neural Information Processing Systems},
  year = {2019},
}

@inproceedings{russac2021selfconcordant,
  author = {Russac, Yoan and Faury, Louis and Capp{\'e}, Olivier and Garivier, Aur{\'e}lien},
  title = {Self-Concordant Analysis of Generalized Linear Bandits with Forgetting},
  booktitle = {International Conference on Artificial Intelligence and Statistics},
  year = {2021},
}

@inproceedings{suk2022tracking,
  author = {Suk, Joe and Kpotufe, Samory},
  title = {Tracking Most Significant Arm Switches in Bandits},
  booktitle = {Conference on Learning Theory},
  year = {2022},
}

@inproceedings{suk2023tracking,
  author = {Suk, Joe and Kpotufe, Samory},
  title = {Tracking Most Significant Shifts in Nonparametric Contextual Bandits},
  booktitle = {Advances in Neural Information Processing Systems},
  year = {2023},
}

@article{wang2025adaptivity,
  author = {Wang, Yining},
  title = {On Adaptivity in Nonstationary Stochastic Optimization with Bandit Feedback},
  journal = {Operations Research},
  volume = {73},
  number = {2},
  year = {2025},
}

@inproceedings{wei2021nonstationary,
  author = {Wei, Chen-Yu and Luo, Haipeng},
  title = {Non-stationary Reinforcement Learning without Prior Knowledge: an Optimal Black-box Approach},
  booktitle = {Conference on Learning Theory},
  year = {2021},
}

@article{zhao2021bandit,
  author = {Zhao, Peng and Wang, Guanghui and Zhang, Lijun and Zhou, Zhi-Hua},
  title = {Bandit Convex Optimization in Non-stationary Environments},
  journal = {Journal of Machine Learning Research},
  volume = {22},
  number = {125},
  year = {2021},
}

@inproceedings{neu2015explore,
  author = {Neu, Gergely},
  title = {Explore no more: Improved high-probability regret bounds for non-stochastic bandits},
  booktitle = {Advances in Neural Information Processing Systems},
  year = {2015},
}
